\documentclass{article}
\usepackage{amsmath, amsthm, amssymb}
\usepackage{microtype}
\usepackage{graphicx}
\usepackage{subcaption}
\usepackage{booktabs} % for professional tables
\usepackage[table]{xcolor}
\usepackage{tabularx}
\usepackage{booktabs}
\usepackage{multirow}
\usepackage{colortbl}
\usepackage{makecell}
\usepackage{array} 
\definecolor{lightblue}{RGB}{230,240,255}
\definecolor{lightpurple}{RGB}{240,230,255}
\definecolor{deeppurple}{RGB}{180,100,200}
\usepackage{hyperref}

\usepackage[preprint]{icml2026}
\usepackage{amsmath}
\usepackage{amssymb}
\usepackage{mathtools}
\usepackage{amsthm}

\usepackage{multirow}
\usepackage[capitalize,noabbrev]{cleveref}

\theoremstyle{plain}
\newtheorem{theorem}{Theorem}[section]
\newtheorem{proposition}[theorem]{Proposition}
\newtheorem{lemma}[theorem]{Lemma}

\theoremstyle{definition}

\theoremstyle{remark}

\usepackage[textsize=tiny]{todonotes}

\icmltitlerunning{Smooth Operator: Smooth Verifiable Reward Activates Spatial Reasoning Ability of Vision-Language Model}

\begin{document}

\twocolumn[
  \icmltitle{Smooth Operator: Smooth Verifiable Reward Activates Spatial Reasoning Ability of Vision-Language Model}
  \begin{center}
    \vspace{-10pt} % 稍微收缩标题和作者之间的距离
    {\textbf{
      Siwen Jiao$^{1,2}$, 
      Tianxiong Lv$^{1\dagger\S}$,
      Kangan Qian$^{3}$,
      Chenxu Zhao$^{1}$, 
      Xiuyuan Zhu$^{1}$, 
      Tianlun Li$^{1}$, \\ 
      Xiaolong Cheng$^{1}$, 
      Jinyu Li$^{1}$, 
      Zhihao Liao$^{1}$, 
      Yang Cai$^{1}$ \\
    }}
    \vspace{6pt} % 作者与机构之间的间距
    {\itshape
      $^{1}$Amap, Alibaba Group \\
      $^{2}$National University of Singapore  \quad
      $^{3}$Tsinghua University
    }\\
    \vspace{6pt} % 作者与机构之间的间距
    \printAffiliationsAndNotice{}

    % 建议放上主要联系人的邮箱，增加专业感
  \end{center}
  %\icmlcorrespondingauthor{Siwen Jiao}{jiaosiwen.jsw@alibaba-inc.com}
  \vskip 0.3in
]

\begingroup
\renewcommand{\thefootnote}{\fnsymbol{footnote}}
\footnotetext[2]{Project Leader}
\footnotetext[4]{Corresponding Author}
\endgroup



%\printAffiliationsAndNotice{}  % no special notice (required even if empty)

\begin{abstract}

Vision-Language Models (VLMs) face a critical bottleneck in achieving precise numerical prediction for 3D scene understanding. Traditional reinforcement learning (RL) approaches, primarily based on relative ranking, often suffer from severe reward sparsity and gradient instability, failing to effectively exploit the verifiable signals provided by 3D physical constraints. Notably, in standard GRPO frameworks, relative normalization causes “near-miss” samples (characterized by small but non-zero errors) to suffer from advantage collapse. This leads to a severe data utilization bottleneck where valuable boundary samples are discarded during optimization. To address this, we introduce the Smooth Numerical Reward Activation (SNRA) operator and the Absolute-Preserving GRPO (AP-GRPO) framework. SNRA employs a dynamically parameterized Sigmoid function to transform raw feedback into a dense, continuous reward continuum. Concurrently, AP-GRPO integrates absolute scalar gradients to mitigate the numerical information loss inherent in conventional relative-ranking mechanisms. By leveraging this approach, we constructed Numerical3D-50k, a dataset comprising 50,000 verifiable 3D subtasks. Empirical results indicate that AP-GRPO achieves performance parity with large-scale supervised methods while maintaining higher data efficiency, effectively activating latent 3D reasoning in VLMs without requiring architectural modifications.

\end{abstract}
\section{Introduction}

Vision-Language Models (VLMs) excel in 2D visual understanding, often rivaling or exceeding human accuracy across diverse benchmarks. However, in genuine 3D scene understanding, these models reveal a critical weakness: they struggle with precise numerical estimation and geometric reasoning, such as depth perception, inter-object distance calculation, and object sizing \cite{chen2024spatialvlm, yang2024depth, liu2025general, yang2025thinking, cheng2024spatialrgpt}. This limitation stems from a reliance on 2D visual encoders that lack explicit geometric priors and depend instead on statistical patterns from planar images.

\begin{figure}
    \centering
    \includegraphics[width=1\linewidth]{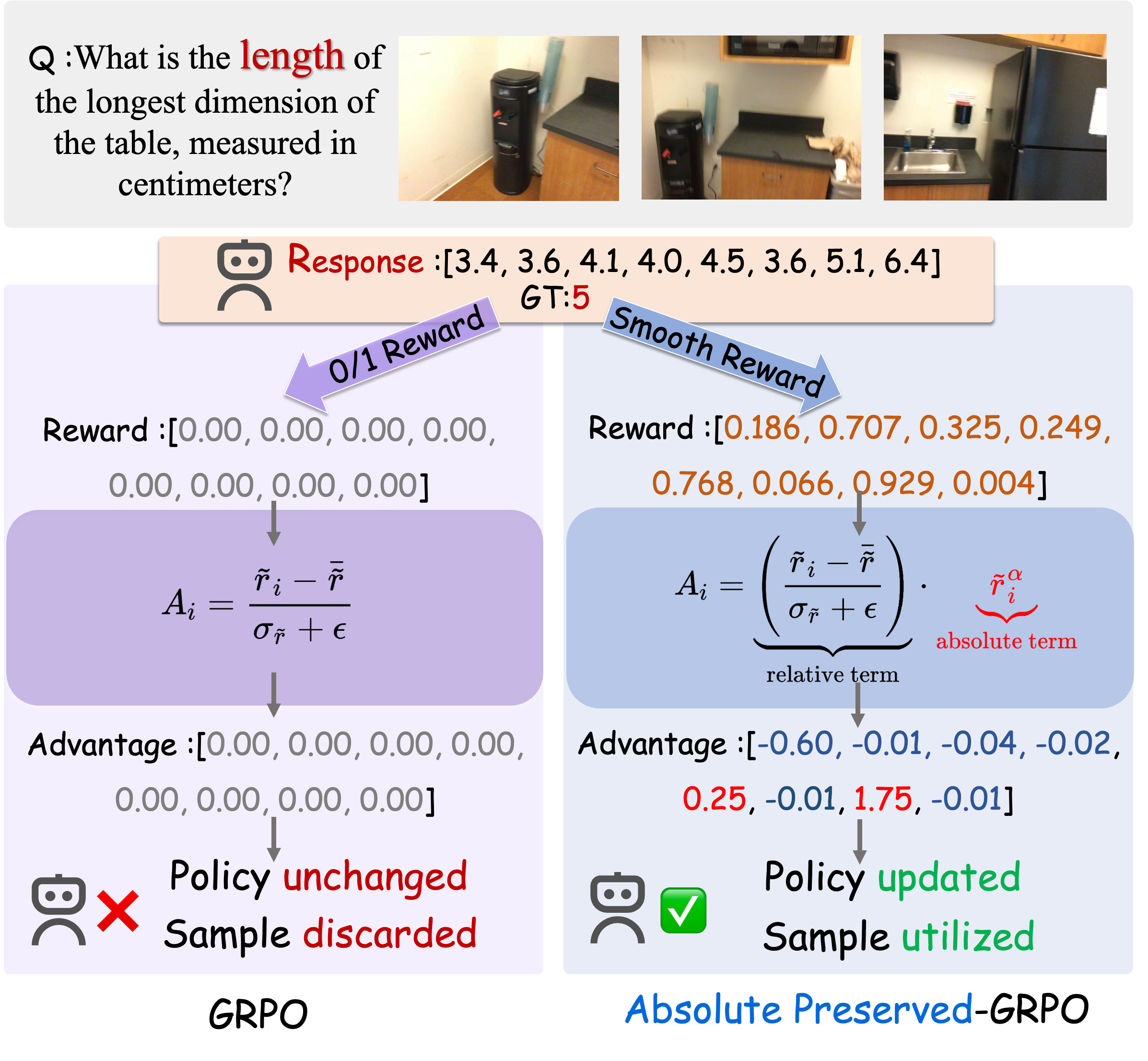}
    \caption{Standard GRPO vs. AP-GRPO(Ours). Standard GRPO (left) assigns binary rewards, yielding zero advantages for near-correct responses and discarding all samples, resulting in no policy update. In contrast, AP-GRPO (right) preserves absolute precision through SNRA smoothing, producing dense rewards and non-zero advantages modulated by the absolute term, thereby enabling effective policy updates and high sample utilization for precise numerical perception.}
    \label{fig:intro}
\end{figure}
Existing approaches to enhancing 3D capabilities in vision-language models (VLMs) fall into two main categories. The first involves feature augmentation, which incorporates external geometric representations (e.g., depth maps, point clouds, or specialized latent features from encoders like VGGT) as additional inputs or supervision \cite{wang2025vggt, fan2025vlm, cheng2024spatialrgpt, zhu2025llava}. While effective, these methods introduce substantial computational overhead, strong reliance on auxiliary 3D modalities or pre-trained geometric models, and potential disruption of the original VLM alignment.
The second strategy relies on large-scale, diverse training datasets to implicitly acquire 3D geometric understanding through extensive supervised or self-supervised training \cite{chen2024spatialvlm, jia2024sceneverse, daxberger2025mm, cheng2024spatialrgpt}. However, such approaches demand enormous data volumes for robust generalization and incur prohibitive training costs, limiting their practicality in resource-constrained settings.

In this work, we argue that verifiable numerical signals grounded in physical laws offer a promising yet underexplored pathway for bridging the gap between 2D perception and 3D understanding. However, effectively leveraging such signals via reinforcement learning (RL) faces several challenges. First, conventional preference optimization frameworks \cite{shao2024deepseekmath} rely primarily on relative comparisons rather than absolute, verifiable scalar feedback, compressing fine-grained numerical information into coarse preference signals. Second, traditional RL often adopts near-binary $0$--$1$ rewards, resulting in extremely sparse supervision and unstable gradients.

These limitations are particularly evident in group-relative methods such as Group Relative Policy Optimization (GRPO) \cite{shao2024deepseekmath}. Under sparse rewards, samples with small but non-zero numerical errors may yield near-zero advantages after group-wise normalization, causing many ``near-miss'' samples to contribute little to policy updates, especially during early training.To address these issues, we introduce two complementary components: the Smooth Numerical Reward Activation Operator (SNRA) and Absolute-Preserving GRPO (AP-GRPO). SNRA is a dynamically parameterized sigmoid-based reward transformation that maps raw verifiable signals, including unbounded squared errors and discrete graded scores, into a dense reward space within $[0,1]$. This yields smoother and more informative gradients in early optimization while progressively enforcing stricter physical tolerances as training proceeds.Building on this transformed reward signal, we propose Absolute-Preserving GRPO (AP-GRPO), a modification of standard GRPO tailored for physics-grounded RL. Unlike pure group-relative normalization, which compresses absolute reward scale information, AP-GRPO introduces an auxiliary absolute reward component aligned with task-specific physical tolerances. This component serves as a corrective reference that anchors group-normalized advantages to a physically meaningful scale, alleviating the loss of scalar gradient information while retaining GRPO’s variance reduction benefits.

Through the combined effect of SNRA and AP-GRPO, previously underutilized boundary samples with small numerical errors are reactivated, improving data efficiency. To fully leverage this, we construct Numerical3D-50k, a compact dataset of roughly 50k samples derived from complex 3D spatial queries with verifiable numerical (e.g., depth, distance, size) and relational (e.g., direction, spatial ordering) subtasks, and train our models on it. Notably, our method preserves the standard VLM architecture without adding encoders, modalities, or structural modifications.

The main contributions of this work are summarized as follows:
\begin{itemize}
    \item \textbf{Numerical3D-50k Dataset.} A compact dataset of approximately 50k samples derived from existing 3D benchmarks, decomposed into verifiable numerical and discrete relational subtasks with deterministic, physics-grounded supervision.
    \item \textbf{SNRA and AP-GRPO.} A dynamic sigmoid-based reward activation operator for densifying sparse physical signals, and a modified GRPO framework that partially preserves absolute reward information through a hybrid advantage formulation while maintaining GRPO’s variance control properties \cite{shao2024deepseekmath}.
    \item \textbf{Empirical Evidence on Numerical Perception.} Experiments demonstrate improved data efficiency by enabling effective learning from near-miss samples, achieving competitive performance with substantially larger-scale training and feature-augmentation baselines using only 50k samples, without architectural modifications.
\end{itemize}

\section{Related Works}
\subsection{Visual Spatial Understanding and Reasoning}
Despite strong progress of large vision-language models (LVLMs) on many visual tasks~\cite{liu2024llavaonevision, bai2025qwen2, chen2024internvl, wang2024qwen2, chen2024spatialvlm}, multiple benchmarks~\cite{yang2025thinking, ray2024sat, cheng2024spatialrgpt} show that spatial understanding and reasoning remain challenging. Prior work improves spatial capability from two complementary angles, namely spatial understanding (perceiving geometry, depth, and relations) and spatial reasoning (multi-step inference grounded in spatial cues)~\cite{qian2025agentthink, qian2025lego, qian2025priormotion, qian2025fasionad++, qian2024fasionad}. On the understanding side, Semantic Abstraction~\cite{ha2022semantic} laid an early foundation for open-world 3D scene understanding with 2D VLMs; subsequent efforts include SpatialVLM~\cite{chen2024spatialvlm} for spatial VQA via expert-constructed data, SpatialRGPT~\cite{cheng2024spatialrgpt} for extending RGB to RGB-D using 3D scene graphs, SAT~\cite{ray2024sat} for simulator-based data synthesis, and SpatialBot~\cite{cai2025spatialbot} for tool-assisted depth estimation, while later works~\cite{liu2025general, yang2025thinking, yang2025mmsi, zhu2025llava} further scale spatial perception with larger and more comprehensive datasets. In parallel, reasoning-oriented methods strengthen inference: MVoT~\cite{li2025imagine} integrates multimodal representations into reasoning traces; SpaceR~\cite{ouyang2025spacer} and MindCube~\cite{yin2025spatial} introduce textual cognitive maps and further improve performance with reinforcement learning; SpatialReasoner~\cite{ma2025spatialreasoner} predicts 3D locations and poses as intermediate outputs; and ViLaSR~\cite{wu2025reinforcing} incorporates visual tools and prompting. Recent RL frameworks for 3D reasoning include MetaSpatial~\cite{pan2025metaspatial}, 3D-R1~\cite{huang20253d}, Spatial-SSRL~\cite{liu2025spatial}, and PRISM~\cite{sun2025prism}. Unlike prior studies that rely on relative preference ranking or massive self-supervision, our approach leverages verifiable numerical signals and absolute-preserving optimization to bridge the perceptual gap.

\subsection{VLM Reinforcement Learning with Verifiable Rewards}

Large Vision-Language Models (LVLMs) traditionally scale via architectural refinements and instruction tuning \cite{liu2024llavaonevision, bai2025qwen2}. Inspired by DeepSeek-R1 \cite{guo2025deepseek}, post-training Reinforcement Learning (RL) has recently emerged as a key paradigm to unlock complex reasoning. Within this trend, Reinforcement Learning with Verifiable Rewards (RLVR) substitutes noisy neural reward models with objective, checkable criteria—such as rule-based validation and programmatic evaluators—thereby mitigating reward ambiguity and enhancing stability \cite{thawakar2025llamavo1, pan2025medvlmr1}. Recent frontiers (2025-2026) have evolved from simple outcome verification to process-aware and perception-grounded mechanisms. For instance, Perception-R1 \cite{xiao2025perceptionr1} and Vision-SR1 \cite{li2025visionsr1} introduce explicit perception rewards to penalize hallucinations and incentivize intermediate visual grounding. Concurrently, StructVRM \cite{zhang2025structvrm} and Argos \cite{tan2025argos} employ agentic verifiers to provide sub-question-level feedback, enabling nuanced partial credit over binary signals. Despite these advances, continuous numerical reasoning—crucial for 3D spatial tasks—remains a challenge. Standard relative-ranking algorithms like GRPO \cite{guo2025deepseek, he2025entropygrpo} often treat "near-miss" samples identically to failures, leading to reward sparsity and advantage collapse. Our work addresses these limitations by proposing Smooth Numerical Reward Activation (SNRA) and Absolute-Preserving GRPO (AP-GRPO). By transitioning from discrete logic-checking to a dense, continuous reward continuum, our framework effectively activates the latent spatial reasoning and precise numerical prediction capabilities of VLMs in verifiable 3D environments.

% \begin{figure*}[ht]
%     \centering
%     \includegraphics[width=0.85\linewidth]{sec/images/ap-grpo.png}
%     \caption{Overview of the Absolute-Preserving GRPO (AP-GRPO) framework. For each training instance, \(G\) response trajectories are sampled and evaluated via task-specific verifiers. Rewards are transformed by SNRA into bounded signals \(\tilde{r}_i\), followed by hybrid advantage computation that preserves the dominant absolute component while incorporating lightweight relative normalization for stability.}
%     \label{fig:ap-grpo}
% \end{figure*}
\section{Method}

\begin{figure*}[ht!]
    \centering
    \includegraphics[width=1.0\linewidth]{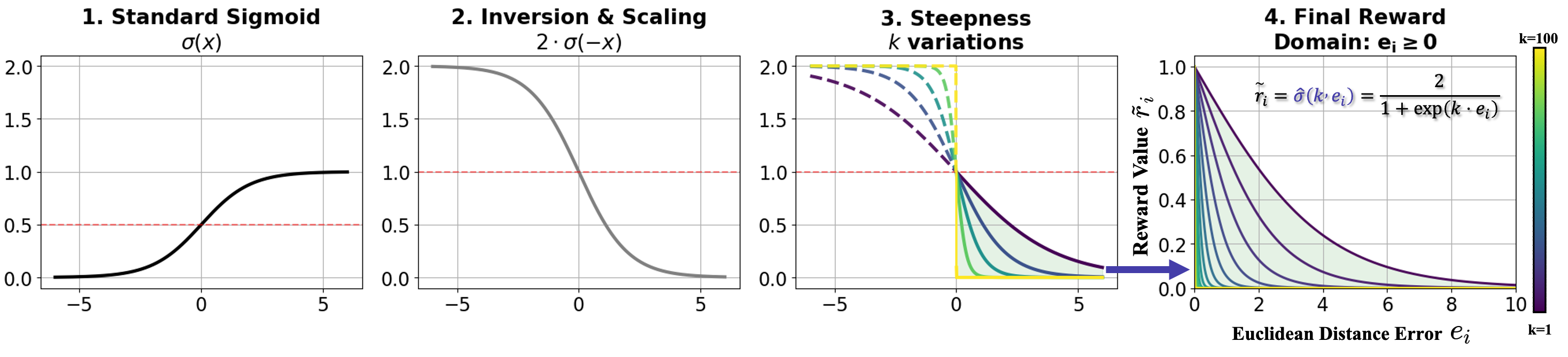}
    \caption{Derivation and mechanism of the SNRA operator. The roadmap depicts the transformation from a standard sigmoid $\sigma(x)$ (Stage 1) to the final reward mapping (Stage 4) via symmetry inversion and amplitude scaling. This formulation anchors the operator at $f(0)=1$ for perfect predictions. The sharpness $k$ acts as a dynamic soft-threshold: lower $k$ provides smooth gradients to encourage exploration, while higher $k$ enforces strict numerical compliance for precision convergence.}
    \label{fig:snra}
\end{figure*}

\subsection{3D Perceptual Subtasks and Numerical3D-50k Dataset}

We decompose 3D scene understanding into two complementary subtasks, each leveraging distinct physical ground-truth signals, and train our method on the curated \textbf{Numerical3D-50k} dataset.

Each sample includes a visual input $x$, a natural-language query $q$, a subtask identifier $\tau \in \mathcal{T}_{\mathrm{num}} \cup \mathcal{T}_{\mathrm{disc}}$, and verifiable geometric annotations. To ensure a unified optimization landscape during reinforcement learning, we map the feedback from both subtask types into a generalized error space $\mathcal{E}$, where $e_i \in [0, +\infty)$ represents the discrepancy between the prediction and the ground truth.

The subtasks are categorized as follows:

(i) \textbf{Numerical subtasks} ($\tau \in \mathcal{T}_{\mathrm{num}}$), requiring precise scalar outputs (e.g., depth, distance, size). The raw error is calculated as the squared physical distance $e_i = (\hat{y}_i - y)^2$. For invalid outputs (e.g., parsing failures), a task-specific maximum penalty $e_{\mathrm{max}}$ is assigned.

(ii) \textbf{Discrete relational subtasks} ($\tau \in \mathcal{T}_{\mathrm{disc}}$), involving categorical spatial reasoning (e.g., direction, ordering). Instead of directly assigning a reward score, we define a logical error $e_i$ derived from the task-specific verification $V_\tau(o_i) \in [0,1]$. Specifically, the error is mapped via $e_i = \Phi(1 - V_\tau(o_i))$, where $\Phi$ is a scaling function that ensures discrete failures (score 0) correspond to the high-penalty region of the SNRA operator, while partial credits are mapped to corresponding gradient-sensitive zones.

Subtasks are algorithmically generated for each scene-query pair using deterministic verifiers $V_\tau$, yielding dense and interpretable supervision suitable for reinforcement learning without model modifications.

The \textbf{Numerical3D-50k} dataset is built by integrating high-quality open-source spatial reasoning benchmarks (MindCube \cite{wu2025reinforcing}, VLM-3R instruction data \cite{fan2025vlm}, VSI-590K \cite{cambrian2025cambrians}) and richly annotated 3D datasets (ScanNet~\cite{dai2017scannet}, ScanNet++~\cite{yeshwanth2023scannet++}, ARKitScenes~\cite{dehghan2021arkitscenes}, SR-91k~\cite{cheng2024spatialrgpt}), which support automated generation of precise metric-scale annotations via geometric projection and template synthesis. From the resulting question-answer pairs, we carefully select 50,000 high-quality samples focused on numerical subtasks ($\mathcal{T}_{\mathrm{num}}$) and their verifiable geometric ground truth, providing balanced, precise, and metric-accurate supervision for fine-grained 3D numerical perception.

\subsection{Smooth Numerical Reward Activation Operator (SNRA)}In reinforcement learning for 3D spatial reasoning, directly optimizing raw verification signals poses significant numerical and structural challenges for autoregressive models. The primary verification signal for geometric alignment is typically defined as the squared Euclidean distance between the predicted coordinate $\hat{y}_i$ and the ground truth $y$:\begin{equation}e_i = |\hat{y}_i - y|^2\end{equation}This raw error $e_i$ is inherently unbounded and exhibits high variance across different training stages. In the early phases of exploration, large deviations can produce excessive gradient spikes; conversely, as the model approaches convergence, the diminishing error magnitudes yield vanishingly small gradients, which are insufficient to drive precision adjustments.To address these limitations, we propose the Smooth Numerical Reward Activation (SNRA) operator. As illustrated in Figure \ref{fig:snra}, the operator is derived from a standard sigmoid base $\sigma(x) = (1+e^{-x})^{-1}$ through a systematic transformation. Specifically, we apply a symmetry transformation and amplitude scaling to construct $f(x) = 2 \cdot \sigma(-x) = 2/(1+e^x)$, which ensures the reward is monotonically decreasing and anchored at $f(0)=1$ for perfect predictions. By introducing a sharpness parameter $k$ and restricting the domain to the physical error space $e_i \ge 0$, we define the specialized reward activation function $\hat\sigma(k, e_i)$ as:\begin{equation}\hat\sigma(k, e_i) = \frac{2}{1 + \exp(k \cdot e_i)}, \quad k > 0, \ e_i \ge 0, \ \tilde{r}_i \in (0, 1]\end{equation}where $k$ denotes the sharpness parameter that scales the reward's sensitivity. A small $k$ provides smooth gradients to encourage exploration, while a large $k$ facilitates precision convergence by rapidly decaying the reward for minor errors.In our framework, the numerical reward component $\tilde{r}_i$ serves as the foundation for enforcing numerical precision across varied subtasks:\begin{equation}\tilde{r}_i =\begin{cases}\hat\sigma\left( k(t), e_i^{\mathrm{con}} \right), & \text{for continuous subtasks} \\\hat\sigma\left( k(t), e_i^{\mathrm{disc}} \right), & \text{for discrete subtasks}\end{cases}\end{equation}where $k(t)$ is the dynamic sharpness parameter. The total reward $R_i$ is then constructed as a composite signal:\begin{equation}R_i =(1-\lambda)\tilde{r}_i + \lambda \cdot r_{fmt}\end{equation}where $r_{fmt} \in \{0, 1\}$ denotes the binary format reward for structural compliance, and $\lambda$ is the balancing coefficient.

\begin{figure*}[h!]
    \centering
    \includegraphics[width=1\linewidth]{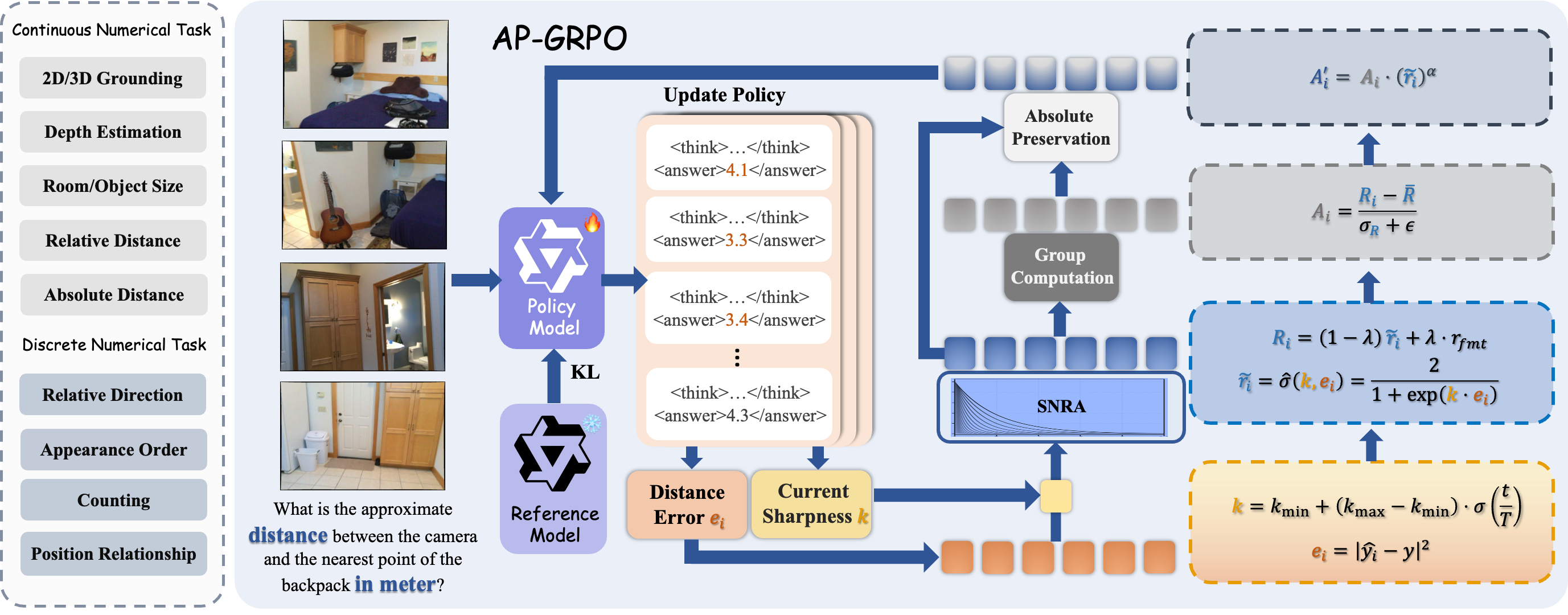}
    \caption{Overview of the AP-GRPO framework.
The method processes continuous (e.g., grounding, depth, size, distance) and discrete numerical tasks (e.g., direction, order, counting, position) via a policy model that generates reasoning and answers. Raw rewards are smoothed by SNRA into dense $R_i'$, then used to compute group-relative advantages modulated by absolute preservation $(R_i')^\alpha$, yielding $A_i'$. Policy updates employ a clipped objective with KL regularization for stable, precise 3D perception.}
    \label{fig:AP-GRPO}
\end{figure*}

\subsection{Dynamic Sharpness Scheduling}

The efficacy of SNRA is governed by the sharpness $k(t)$, which controls reward sensitivity to the error $e_i$. Instead of treating $k$ as a fixed hyperparameter, we introduce Dynamic Sharpness Scheduling, a non-linear gradient amplifier that gradually shifts optimization from global exploration to local precision refinement.

Early in training, a small $k(t)$ produces a diffused reward surface with broad gradient support, providing non-zero feedback even when predictions are far from the target and preventing stalls during global coordinate search. As training proceeds, we monotonically increase $k(t)$ to sharpen the reward peak around $e_i=0$, concentrating gradients on subtle numerical errors and promoting high-precision alignment.

We implement this coarse-to-fine curriculum with a centered, smoothed sigmoid schedule:
\begin{equation}
k(t) = k_{\min} + (k_{\max}-k_{\min}) \cdot \sigma\!\left(s\left(\frac{t}{T}-\tau\right)\right),
\end{equation}
where $t$ and $T$ are the current and total steps. $\tau \in (0,1)$ sets the transition center and $s>0$ controls its steepness. We use $\tau=0.5$ and $s=10$. Thus, $k(t)$ stays near $k_{\min}$ early, rises around $t \approx \tau T$, and approaches $k_{\max}$ late in training. In all experiments, $k_{\min}=1.0$ and $k_{\max}=100.0$.

Coupling sharpness to training progress induces \emph{reward hardening}. As $k \to \infty$, the smooth surrogate $\tilde{r}_i$ approaches a binary verification signal:
\begin{equation}
\tilde{r}_i \xrightarrow{k \to \infty} \mathbb{I}_{e_i} =
\begin{cases}
1, & \text{if } e_i = 0, \\
0, & \text{if } e_i > 0.
\end{cases}
\end{equation}
This lets the agent first learn global geometry under permissive rewards, then face increasingly strict precision requirements, balancing exploration stability with accurate reasoning.

\subsection{Absolute-Preserving GRPO (AP-GRPO)}

Standard Group Relative Policy Optimization (GRPO) estimates the advantage function through intra-group standardization:

\begin{equation}
    A_i = \frac{R_i - \bar{R}}{\sigma_{R} + \epsilon},
\end{equation}

where $\bar{R}$ and $\sigma_{R}$ are the mean and standard deviation of rewards within a group of $G$ samples generated from the same query. By enforcing $\sum A_i = 0$, this formulation ensures that the policy update depends only on the relative ranking of samples within the group, increasing the probability of those that outperform the local average.

While relative ranking is effective in high-quality regimes such as language tasks, it encounters severe limitations in 3D spatial reasoning. When all samples in a group receive extremely low or zero rewards (common in physically invalid predictions), the standardization process eliminates absolute scale information and assigns near-zero advantages to all samples. This renders many trajectories effectively wasted, as they contribute no meaningful gradient updates despite containing potentially recoverable directional signals.

To recover absolute numerical optimization signals and prevent the waste of low-reward but informative samples, we propose \textbf{Absolute-Preserving GRPO (AP-GRPO)}. We redefine the advantage function as the product of a relative ranking term and an absolute scaling term:

\begin{equation}
    A_i^\prime = \underbrace{\left( \frac{R_i - \bar{R}}{\sigma_{R} + \epsilon} \right)}_{\text{Relative term}} \cdot \underbrace{\tilde{r}_i^{\alpha}}_{\text{Absolute term}}, \quad \alpha \ge 1,
    \label{eq:apgrpo}
\end{equation}
where $\alpha$ is a hyperparameter controlling the strength of absolute reward modulation. The first term preserves intra-group relative ranking based on the total reward $R_i$, while the second term $\tilde{r}_i^{\alpha}$ scales the update magnitude by absolute accuracy: when $\tilde{r}_i$ is close to $1$, $\tilde{r}_i^{\alpha}\approx 1$, so $A_i^\prime$ behaves similarly to standard GRPO and can exploit fine-grained relative differences for precise optimization; when $\tilde{r}_i$ is small (including near-zero), $\tilde{r}_i^{\alpha}$ strongly suppresses the advantage magnitude, so even top-ranked samples in a group contribute little due to poor absolute accuracy. From an optimization perspective, $\tilde{r}_i^{\alpha}$ serves as a weighting factor that scales gradient magnitudes in proportion to each sample’s physical plausibility, concentrating updates on trajectories with meaningful accuracy while suppressing noisy ranking signals from groups dominated by inaccurate responses.

In our implementation, advantages are computed at the token level, with symmetric clipping applied to maintain numerical stability. Substituting the modulated advantage into the policy objective yields the following clipped surrogate loss:
\begin{equation}
\begin{aligned}
\mathcal{L}(\theta) ={}& \mathbb{E} \Biggl[ \frac{1}{G} \sum_{i=1}^{G} \min \Biggl(
    \rho_i \left( \frac{R_i - \bar{R}}{\sigma_{R} + \epsilon} \tilde{r}_i^{\alpha} \right), \\
& \quad \text{clip}(\rho_i, 1-\epsilon, 1+\epsilon) \left( \frac{R_i - \bar{R}}{\sigma_{R} + \epsilon} \tilde{r}_i^{\alpha} \right)
\Biggr) \Biggr] \\
& - \beta_{\text{KL}} \, D_{\text{KL}}(\pi_{\theta} \| \pi_{\text{ref}}),
\end{aligned}
\end{equation}
where $\rho_i = \frac{\pi_{\theta}(o_i | x, q)}{\pi_{\text{old}}(o_i | x, q)}$ denotes the probability ratio between the current and sampling policies. By incorporating absolute information from SNRA into the advantage computation, AP-GRPO recovers reliable gradient signals from low-reward samples, prevents wasteful updates on inaccurate trajectories, and achieves more stable convergence with superior metric accuracy on 3D perception tasks.

\section{Experiments}

\begin{table*}[ht!]
\centering
\caption{Comprehensive Evaluation of Spatial Intelligence. Table (a) evaluates fine-grained video-based spatial reasoning on VSI-Bench, while Table (b) summarizes performance across diverse 3D spatial reasoning and general multimodal benchmarks.}
\label{tab:comprehensive_spatial_results}

\begin{subtable}{\textwidth}\centering
\footnotesize
\caption{Evaluation Results on VSI-Bench with Data Scale and Training Phases}
\resizebox{\textwidth}{!}{%
\begin{tabularx}{\textwidth}{X  >{\centering\arraybackslash}p{0.60cm} >{\centering\arraybackslash}p{0.99cm} >{\centering\arraybackslash}p{0.85cm} *{8}{>{\centering\arraybackslash}p{0.80cm}}}
\toprule
\multirow{3}{*}{\textbf{Methods}} & \multirow{3}{*}{\textbf{\shortstack{Data\\Scale}}} & \multirow{3}{*}{\textbf{\shortstack{Training\\Phases}}} & \multirow{3}{*}{\textbf{Avg.}} & \multicolumn{4}{c}{\textbf{Numerical Question}} & \multicolumn{4}{c}{\textbf{Multiple-Choice Question}} \\
\cmidrule(lr){5-8} \cmidrule(lr){9-12}
& & & & \textbf{Obj. Cnt.} & \textbf{Abs. Dist.} & \textbf{Obj. Size} & \textbf{Room Size} & \textbf{Rel. Dist.} & \textbf{Rel. Dir.} & \textbf{Route Plan} & \textbf{Appr. Ord.} \\
\midrule
\multicolumn{12}{c}{Proprietary and Open-source General Models} \\
\midrule
GPT-4o \cite{hurst2024gpt}          & - & - & 34.0 & 46.2 & 5.3  & 43.8 & 38.2 & 37.0 & 41.3 & 31.5 & 28.5 \\
Gemini-1.5 Pro \cite{team2024gemini}  & - & - & 45.4 & 56.2 & 30.9 & 64.1 & 43.6 & 51.3 & 46.3 & 36.0 & 34.6 \\
LLaVA-OneVision-7B  & - & - & 32.4 & 47.7 & 20.2 & 47.4 & 12.3 & 42.5 & 35.2 & 29.4 & 24.4 \\
LLaVA-Video-7B      & - & - & 35.6 & 48.5 & 14.0 & 47.8 & 24.2 & 43.5 & 42.4 & 34.0 & 30.6 \\
Qwen2.5-VL-7B       & - & - & 32.7 & 34.5 & 19.4 & 47.6 & 40.8 & 32.8 & 24.5 & 32.5 & 29.4 \\
\midrule
% 使用浅灰色背景强调单阶段微调部分
\rowcolor[gray]{0.92}
\multicolumn{12}{c}{Open-source Models with \textbf{Standard Architecture} Finetuning} \\
\midrule
SAT-7B              & 175K  & 1 (S) & -    & -    & -    & -    & -    & 47.3 & 41.1 & 37.1 & 36.1 \\
SpatialLadder-3B  & 26K & 3 (S+R) & 45.7  & -    & -    & -    & - & -    & -    & -    & - \\
InternVL-Spatial-8B & 291K & 1 (S) & 52.3   & 68.7 & 40.9 & 63.1 & 54.3 & 47.7 & -    & 29.9 & 60.5 \\
SpaceR-7B           & 151K & 2 (S+R)  & 43.5 & 61.9 & 28.6 & 60.9 & 35.2 & 38.2 & 46.0 & 31.4 & 45.6 \\
ViLaSR-7B           & 81K  & 2 (S+R) & 45.4 & 63.5 & 34.4 & 60.6 & 30.9 & 48.9 & 45.2 & 30.4 & 49.2 \\
VST-3B-SFT          & 4.1M & 1 (S) & 57.9 & 69.3 & 45.4 & 71.8 & 62.4 & 59.0 & 46.0 & 38.7 & \textbf{70.2} \\
VST-7B-SFT          & 4.1M & 1 (S) & \textbf{60.6} & \textbf{72.0} & \textbf{44.4} & \textbf{74.3} & 68.3 & 59.7 & \textbf{55.8} & \textbf{44.9} & 65.2 \\
\textbf{SmoothOp-3B (Ours)}  & \textbf{50K}  & 1 (R) & 53.9 & 70.1 & 42.8 & 69.9 & 38.3 & 60.4 & 45.4 & 38.7 & 65.5 \\
\textbf{SmoothOp-7B (Ours)}  & \textbf{50K} & 1 (R) & \underline{60.0} & 71.1 & 41.4 & 72.2 & \textbf{70.3} & \textbf{61.7} & 55.1 & 42.8 & 65.4 \\
\bottomrule
\end{tabularx}%
}
\label{tab:vsi_bench_full}
\end{subtable}

\vspace{6pt}

\begin{subtable}{\textwidth}
\centering
\small  % 字体大小：清晰、可读，比 \footnotesize 大
\caption{Evaluation Results on Other Spatial Intelligence Benchmarks.}
\label{tab:performance_comparison}
\begin{tabularx}{\textwidth}{X >{\centering\arraybackslash}p{1.9cm} >{\centering\arraybackslash}p{1.9cm} c c c c }
\toprule
\textbf{Model} & \textbf{\makecell{Data\\Scale}} & \textbf{\makecell{Training\\Phases}} & \textbf{CV-Bench} & \textbf{MMSI} & \textbf{MindCube} & \textbf{BLINK} \\
\midrule
\multicolumn{7}{c}{Proprietary Models and Open-source General Models}\\
\midrule
Gemini-2.5-Pro\cite{comanici2025gemini} & -- & -- & 53.5 & 38.0 & 57.6 & 70.6 \\
GPT-4o & -- & -- & 76.0 & 30.3 & 53.4 & 65.9 \\
LLava-OneVision-7B & -- & -- & 61.9 & 31.0 & 34.7 &   48.2 \\
Qwen2.5-VL-3B & -- & -- & 71.4 & 28.6 & 37.6 & 47.6 \\
Qwen2.5-VL-7B & -- & -- & 75.4 & 26.8 & 36.0  & 56.4 \\
InternVL3-8B & -- & -- & 42.1 & 28.0 & 41.5 & 55.5 \\

\midrule
\rowcolor[gray]{0.92}
\multicolumn{7}{c}{Open-source Models with \textbf{Standard Architecture} Finetuning} \\
\midrule
SpatialLadder-3B & 26K & 3 (S+R) & 73.7 & 26.1 & \textbf{43.4} & 56.9 \\
SPAR-8B & 7M & 2 (S) & 80.7 & -- & -- &43.9 \\
SpaceR-7B & 151K &2 (S+R) & 74.8 & 20.1 & 37.9  & 55.4 \\
ViLaSR-7B & 81K & 2 (S+R) & 81.8 & 30.2 & 35.1 & 56.2 \\
VST-3B-SFT & 4.1M & 1 (S) & 84.4 & 30.2 & 35.9 & 59.1 \\
VST-7B-SFT & 4.1M & 1 (S) & 85.5 & 32.0 & 39.7 & 62.1 \\
\textbf{SmoothOp-3B (Ours)} & \textbf{50K} & 1 (R) & 
\textbf{84.9}\textcolor{blue!78!black}{(+18.9\%)} & 
\textbf{32.2}\textcolor{blue!78!black}{(+12.6\%)} & 
39.4\textcolor{blue!71!black}{(+4.8\%)} & 
\textbf{56.9}\textcolor{blue!78!black}{(+19.5\%)} \\
\textbf{SmoothOp-7B (Ours)} & \textbf{50K} & 1 (R) & 
\textbf{86.4}\textcolor{blue!78!black}{(+14.6\%)} & 
\textbf{33.7}\textcolor{blue!78!black}{(+25.7\%)} & 
39.6\textcolor{blue!71!black}{(+10.0\%)} & 
\textbf{62.6}\textcolor{blue!78!black}{(+11.0\%)}\\
\bottomrule
\multicolumn{7}{l}{\footnotesize}
\end{tabularx}
\end{subtable}
\end{table*}

\begin{table*}
\centering
\caption{\textbf{Roadmap of optimization mechanisms on VSI-Bench.} We evaluate the evolution of our framework across different reward designs and scheduling strategies. Convergence Steps denotes the training iterations required to reach 95\% accuracy; Advantage Variance represents the stability of the policy gradient signal.}
\label{tab:ablation_reward}
\small 
% 使用 tabularx，设置总宽度为 \textwidth
% l: 第一列左对齐；X: 第二列自动伸展填充；c: 后三列居中
\setlength{\tabcolsep}{15.7pt} 
\begin{tabularx}{\textwidth}{llccc}
\toprule
\textbf{ID} & \textbf{Optimization Mechanism} & \textbf{\makecell{Convergence\\Steps ($\downarrow$)}} & \textbf{\makecell{Advantage\\Variance ($\downarrow$)}} & \textbf{\makecell{Average\\Accuracy (\%)}} \\
\midrule
$S_0$ & SFT & N/A & N/A & 49.7 \\
$S_1$ & GRPO with Binary 0-1 Reward & 310 & 0.152 & 54.4 \\
$S_2$ & GRPO with SNRA (Fixed $k$) & 390 & 0.165 & 53.8 \\
$S_3$ & \textbf{AP-GRPO} with SNRA (Fixed $k$) & 250 & 0.148 & 55.4 \\
$S_4$ & \textbf{AP-GRPO} with SNRA (Linear Scheduling) & 150 & 0.142 & 58.9 \\
\rowcolor[gray]{0.9}
$S_5$ & \textbf{AP-GRPO} with \textbf{SNRA (Sigmoid Scheduling)} & \textbf{130} & \textbf{0.138} & \textbf{60.0} \\
\bottomrule
\end{tabularx}
\end{table*}

\subsection{Experimental Setup}

\noindent \textbf{Implementation Details.}We employ Qwen2.5-VL-3B and 7B \cite{wang2024qwen2} as our base vision-language models. Our training adopts a streamlined \textbf{one-stage reinforcement learning} paradigm using Absolute-Preserving GRPO (AP-GRPO) on the Numerical3D-50k dataset, bypassing the need for massive supervised pre-training or Chain-of-Thought (CoT) cold-start. For each query, we sample $G=8$ trajectories. We use the AdamW optimizer with a learning rate of $1\times 10^{-6}$ and weight decay of $0.01$. The KL divergence penalty is set to $\beta_{\mathrm{KL}}=0.02$. For the SNRA operator, we implement a sigmoid sharpness schedule with $k_{\tau,\min}=1.0$ $k_{\tau,\max}=100.0$ and
$\lambda=0.1$. The advantage clipping and absolute modulation coefficient are set to $c=1.5$ and $\alpha=1.0$, respectively. Training is conducted on 8$\times$ NVIDIA A100 GPUs for a single epoch. The original VLM architecture is fully preserved without any additional encoders or auxiliary input modalities.

\noindent \textbf{Evaluation Benchmarks.}We evaluate \textbf{SmoothOp} across five benchmarks spanning the spectrum from numerical precision to high-level geometric reasoning. These include \textbf{VSI-Bench}~\cite{cambrian2025cambrians} for video-based 3D layout reasoning, \textbf{CV-Bench}~\cite{tong2024cambrian1} for foundational 2D/3D perception, \textbf{MMSI-Bench}~\cite{yang2025mmsi} for multi-perspective integration, \textbf{MindCube}~\cite{yin2025spatialmentalmodelinglimited} for occluded scene reconstruction, and \textbf{BLINK}~\cite{fu2024blink} for ego-allocentric perspective-taking. This diverse suite comprehensively assesses the model's spatial reasoning and generalization capabilities.

\subsection{Main Results}

\noindent \textbf{Overall Performance.}
As shown in Table~\ref{tab:vsi_bench_full} and Table~\ref{tab:performance_comparison}, our method exhibits competitive performance across multiple spatial benchmarks compared to leading general-purpose models such as Qwen2.5-VL and LLaVA-OneVision. These results suggest that incorporating smooth, verifiable numerical rewards can provide more precise guidance for spatial perception than standard vision-language objectives.

\noindent \textbf{Effectiveness of Sparse Numerical Supervision.}
A notable observation is the comparison between our method and VST~\cite{yang2025visual} regarding supervision scale. While VST is trained on 4.5M supervised spatial samples, SmoothOp achieves comparable results—and slight improvements on CV-Bench (86.4\% vs. 85.5\%)—using \textbf{50K} targeted samples. This suggests that the SNRA operator and AP-GRPO effectively extract geometric information from numerical rewards, potentially reducing the reliance on exhaustive supervised pre-training for coordinate distribution learning.

\noindent \textbf{Generalization to 3D Tasks.}
Results on MindCube and BLINK suggest that late-stage numerical alignment via reward hardening improves broader 3D reasoning, boosting occluded reconstruction and perspective-taking. This indicates that precise numerical supervision promotes more robust 3D geometric representations beyond the training objective.

\subsection{Roadmap of Reward and Architecture Synergy}

To quantitatively analyze training dynamics, we define \textbf{Advantage Variance (Adv. Var.)} as the group-wise variance of advantage estimates $\sigma^2_A = \frac{1}{G} \sum_{i=1}^G (\hat{A}_i - \bar{A})^2$, which serves as a proxy for gradient stability. Correspondingly, \textbf{Convergence Steps (Conv.)} are quantified by $T_{conv} = \min \{ t \mid \text{Acc}_t \ge 0.95 \cdot \text{Acc}_{max} \}$, assessing the optimization efficiency. Starting from the SFT baseline ($S_0$, 49.7\%), Table~\ref{tab:ablation_reward} and Fig.~\ref{fig:placeholder} illustrate the framework's evolution through three strategic stages, clarifying the interaction between reward shaping and RL architecture.

\noindent \textbf{Stage 1: Signal Dilution in Vanilla GRPO ($S_1 \to S_2$).}

Directly applying the \textbf{SNRA} operator within standard GRPO ($S_2$) yields 53.8\% accuracy, a slight dip from the binary baseline ($S_1$, 54.4\%). As shown in Fig.~\ref{fig:placeholder}, $S_2$ exhibits the highest \textbf{Adv. Var.} (0.165) and a protracted \textbf{Conv.} period of 310 steps. This suggests an architectural limitation: vanilla GRPO's group normalization tends to "wash out" fine-grained numerical differences from SNRA. Consequently, absolute precision gains are partially discarded as noise while relative ranking dominates the gradient, leading to unstable oscillations.

\noindent \textbf{Stage 2: Restoring Absolute Signals via AP-GRPO ($S_2 \to S_3$).}

The introduction of \textbf{Absolute-Preserving (AP-GRPO)} ($S_3$) marks a significant improvement, elevating accuracy to 55.4\% and accelerating convergence by 2.6$\times$ (250 steps). $S_3$ escapes the oscillation zone early, validating that smooth numerical rewards require an explicit absolute reference to be effective. By ensuring that subtle coordinate refinements are directly reflected in policy gradients without being neutralized, this synergy stabilizes optimization (\textbf{Adv. Var.} 0.148) and enables the transition from coarse estimation to precise reasoning.

\noindent \textbf{Stage 3: Dynamic Scheduling as Late-stage Momentum ($S_3 \to S_5$).}

Finally, we observe that fixed-sharpness models ($S_3$) often hit a performance ceiling, as a static reward landscape provides diminishing discriminative signals when the model nears the target. In contrast, dynamic scheduling injects sustained optimization push. The \textbf{Sigmoid schedule ($S_5$, 60.0\%)} achieves the best results with the lowest variance (\textbf{0.138}). By transitioning from global coordinate searching to millimeter-level alignment, this non-linear curriculum helps break late-stage plateaus and provides the model with the necessary "upside potential" to reach its highest performance bound.

\begin{figure}[t!]
    \centering
    \includegraphics[width=1\linewidth]{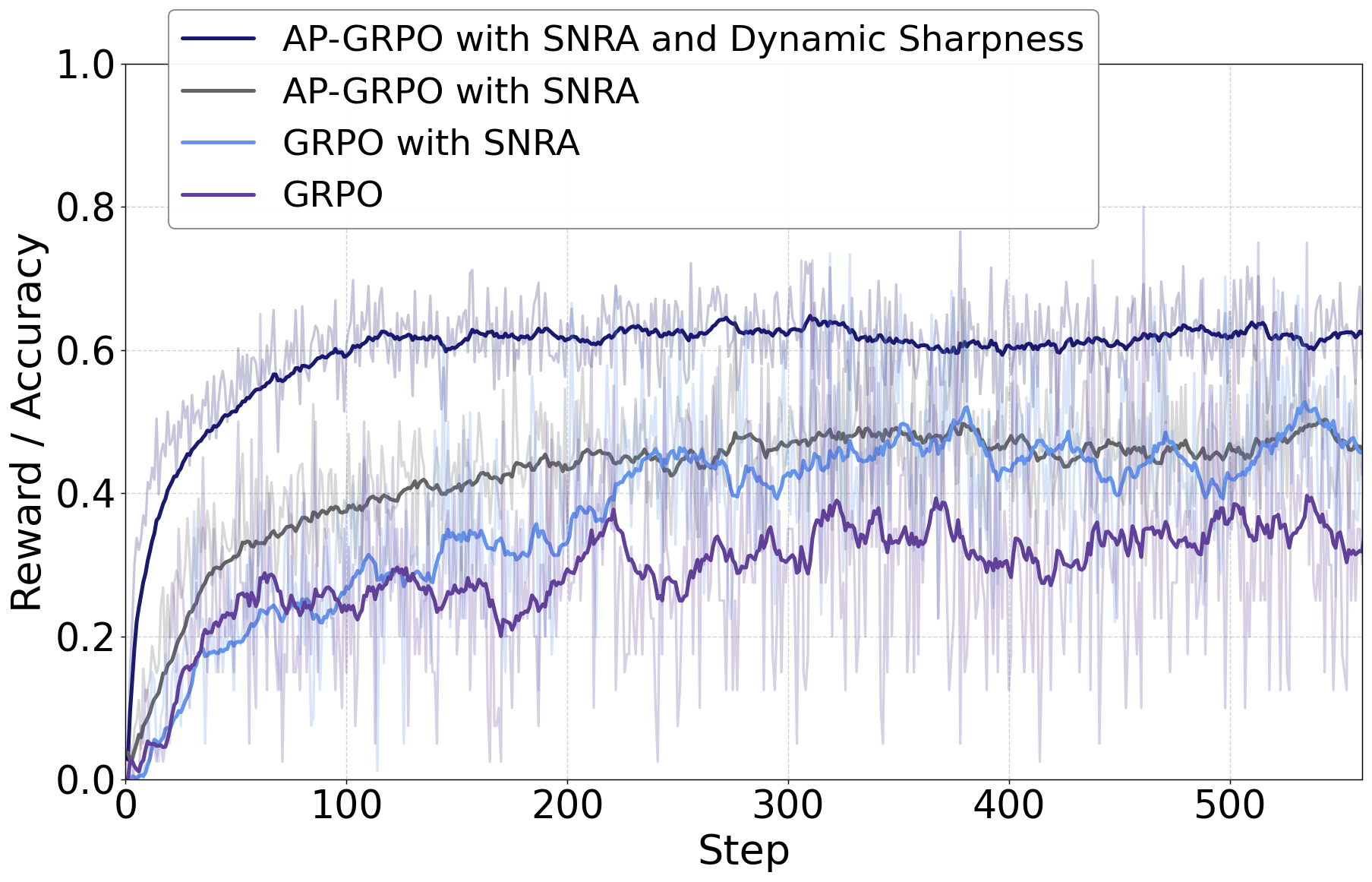}
    \caption{Component-level ablation study. The synergy of SNRA and Dynamic Sharpness allows AP-GRPO to surpass standard GRPO variants, validating the necessity of geometric priors and reward landscape refinement for precise 3D reasoning.}
    \label{fig:placeholder}
\end{figure}

\begin{table}[t!]
\centering
\caption{Comparison of advantage estimation strategies on VSI-Bench. }
\label{tab:ablation_grpo}
\footnotesize
\begin{tabular}{lcc}
\toprule
\textbf{Advantage Formulation} & \textbf{Sync. Type} & \textbf{Avg. Acc (\%)} \\
\midrule
Standard GRPO                  & Relative only      & 54.4 \\
Pure Absolute Scalar           & Absolute only      & 43.1 \\
\rowcolor[gray]{0.9}
\textbf{AP-GRPO (Ours)}        & \textbf{Rel. + Abs.} & \textbf{60.0} \\
\bottomrule
\end{tabular}
\end{table}

% --- Table 3: Data Efficiency Comparison ---
\begin{table}[t!]
\centering
\caption{Data efficiency and scaling comparison on VSI-Bench.}
\label{tab:data_efficiency}
\footnotesize
% \tabcolsep 控制列间距，设小一点防止溢出
\setlength{\tabcolsep}{4pt} 
\begin{tabularx}{\columnwidth}{l *{3}{>{\centering\arraybackslash}X}}
\toprule
& \multicolumn{3}{c}{\textbf{Acc. (\%) @ Training Samples}} \\
\cmidrule(lr){2-4}
\multirow{-2.5}{*}{\textbf{Method}} & \textbf{10k} & \textbf{20k} & \textbf{50k} \\
\midrule
Standard GRPO   & 44.9 & 48.6 & 54.4 \\
\rowcolor[gray]{0.9}
\textbf{AP-GRPO (Ours)} & 51.4 & 57.6 & \textbf{60.0} \\
\bottomrule
\end{tabularx}
\end{table}

\subsection{Ablation Studies}

We conduct ablations on the VSI-Bench validation set to assess component necessity and scaling behavior.

\textbf{Advantage Formulation: Relative vs. Absolute.} 
AP-GRPO outperforms standard GRPO and purely absolute variants by formulating the advantage as a multiplicative product of relative rank and absolute scalar (Table~\ref{tab:ablation_grpo}). While standard GRPO collapses precision gaps into coarse ranks and pure absolute approaches suffer from high gradient instability (43.1\% accuracy), AP-GRPO preserves reward magnitudes for fine-grained steering while maintaining group-wise stability. The 16.9\% gain over the absolute baseline confirms that this synergy is essential for robust 3D reasoning.

\textbf{Data Efficiency and Scaling.}
AP-GRPO is more sample-efficient and scales better than standard GRPO (Table~\ref{tab:data_efficiency}). With only 10k samples, it reaches 51.4\% accuracy, which nearly matches the baseline performance trained on 50k samples, corresponding to a \textbf{5$\times$ efficiency gain}. We attribute this to leveraging ``near-miss'' boundary samples that are often down-weighted or discarded by relative normalization. Moreover, AP-GRPO exhibits a steeper scaling curve, translating additional data into stronger geometric priors and a consistently larger gap over standard RL.

% \textbf{Dataset Scale-up and Task Complexity.} 
% We investigate the scalability of our method across two dimensions: data volume and task diversity (Table~\ref{tab:scale_up}). 
% \textit{(i) Data Scale:} Our model exhibits a strong scaling trend; notably, with only 50k samples, AP-GRPO achieves performance parity with standard GRPO trained on 10$\times$ more data, directly validating the \textbf{data efficiency} of our approach. 
% \textit{(ii) Task Diversity:} Scaling up from simple depth estimation to integrated 3D spatial queries (Numerical3D-50k full set) leads to emergent reasoning capabilities, suggesting that SNRA and AP-GRPO successfully translate increased task complexity into more robust geometric priors.

% --- Table 2: AP-GRPO vs Standard ---

%\input{sec/x.others}
\section{Conclusion}
We introduce a data-efficient RL framework for 3D spatial reasoning. Through the synergy of SNRA and AP-GRPO, we activate boundary gradients often wasted in standard relative optimization. Our approach achieves competitive performance on spatial intelligence benchmarks matching million-scale baselines using only 50k samples. This demonstrates that coupling verifiable physical signals with absolute-preserving RL can effectively instill robust geometric priors in VLMs without architectural changes.

\nocite{langley00}

\bibliography{example_paper}
\bibliographystyle{icml2026}

%%%%%%%%%%%%%%%%%%%%%%%%%%%%%%%%%%%%%%%%%%%%%%%%%%%%%%%%%%%%%%%%%%%%%%%%%%%%%%%
%%%%%%%%%%%%%%%%%%%%%%%%%%%%%%%%%%%%%%%%%%%%%%%%%%%%%%%%%%%%%%%%%%%%%%%%%%%%%%%
% APPENDIX
%%%%%%%%%%%%%%%%%%%%%%%%%%%%%%%%%%%%%%%%%%%%%%%%%%%%%%%%%%%%%%%%%%%%%%%%%%%%%%%
%%%%%%%%%%%%%%%%%%%%%%%%%%%%%%%%%%%%%%%%%%%%%%%%%%%%%%%%%%%%%%%%%%%%%%%%%%%%%%%
\newpage
\appendix
\onecolumn
\section{Theoretical Analysis of AP-GRPO}
\label{sec:theoretical_analysis}

We provide a theoretical analysis of how absolute reward scaling $r^\alpha$ improves optimization stability and numerical precision in policy learning.

\subsection*{Assumptions}

To ensure generality and analytical rigor, we make the following assumptions throughout this section:
\begin{enumerate}
\item Rewards are normalized using a bounded monotonic transformation $r \leftarrow f(r)$ such that $r \in [0,1]$.

\item The GRPO advantage normalization preserves reward ordering within each group $\mathcal{G}$ when $\sigma_{r,\mathcal{G}} > 0$.

\item Reward noise $\xi_i$ is assumed to be i.i.d., sub-Gaussian, and satisfies a bounded fourth moment $\mathbb{E}[\xi_i^4] \le K$ for some constant $K$.

\item In sparse reward regimes, higher-order variance terms are bounded by $O(\epsilon^{2\alpha-1})$, where $\epsilon = \mathbb{E}_{\mathcal{G}}[r_i] \ll 1$.

\item Policy updates are dominated by within-group preference gradients, making local advantage ordering the primary driver of learning stability.
\end{enumerate}
\subsection{Gradient Directionality and Policy Reweighting}

A key concern when modifying the advantage function in reinforcement learning is whether the resulting update still preserves the qualitative directionality of policy improvement. Although AP-GRPO introduces a non-uniform weighting by the absolute reward term $r^\alpha$, we show that it maintains the sign structure of the group-relative advantage and hence preserves the local preference ordering of trajectories within a group.

\subsubsection{Setup}

Let the standard RL objective be
\[
\mathcal{J}(\theta) = \mathbb{E}_{\tau \sim \pi_\theta}[R(\tau)],
\]
and denote by $\phi(\tau) = \nabla_\theta \log \pi_\theta(\tau)$ the score function.

Let $A_{GRPO}(\tau)$ be the group-normalized advantage used in GRPO, constructed from reward $r(\tau)$ within a group $\mathcal{G}$ as
\[
A_{GRPO}(\tau) = \frac{r(\tau) - \bar r_\mathcal{G}}{\sigma_{r,\mathcal{G}} + \epsilon},
\]
so that within each group we approximately have $\mathbb{E}_\mathcal{G}[A_{GRPO}] \approx 0$ and $\mathbb{E}_\mathcal{G}[A_{GRPO}^2] \approx 1$.

The GRPO update direction (ignoring clipping and other regularization) can be written as
\[
g_{GRPO} \;=\; \mathbb{E}_{\tau \sim \pi_\theta}
\big[ \phi(\tau)\, A_{GRPO}(\tau)\big].
\]

AP-GRPO introduces an absolute reward weighting
\[
w(r) = r^\alpha,\quad r \in [0,1],\;\alpha \ge 1,
\]
and uses the modified gradient
\[
g_{AP} \;=\; \mathbb{E}_{\tau \sim \pi_\theta}
\big[ \phi(\tau)\, A_{GRPO}(\tau)\, w(r(\tau))\big].
\]

\subsubsection{Monotone Reweighting and Local Ordering Preservation}

We first clarify what AP-GRPO can be guaranteed to preserve.

\begin{proposition}[Sign preservation and positive ordering]
Let $\mathcal{G}$ be a GRPO group, and for $\tau \in \mathcal{G}$ define
\[
A(\tau) = A_{GRPO}(\tau), \qquad
\tilde{A}(\tau) = A_{GRPO}(\tau)\, r(\tau)^\alpha.
\]
Assume that reward $r(\tau)$ is monotone in the original advantage, i.e.
\[
A(\tau_1) > A(\tau_2)\quad \Rightarrow\quad r(\tau_1) \ge r(\tau_2).
\]
Then for any $\tau \in \mathcal{G}$:
\begin{enumerate}
    \item $\mathrm{sign}\big(\tilde{A}(\tau)\big) = \mathrm{sign}\big(A(\tau)\big)$.
    \item If $A(\tau_1) > A(\tau_2) > 0$, then $\tilde{A}(\tau_1) \ge \tilde{A}(\tau_2) > 0$.
\end{enumerate}
In particular, AP-GRPO guarantees that trajectories with positive relative advantages maintain their local ranking, ensuring that the best-performing samples within a group receive the strongest positive reinforcement.
\end{proposition}

\begin{proof}
(1) For any $\tau$, $r(\tau)\in[0,1]$ and $\alpha \ge 1$ imply $r(\tau)^\alpha \ge 0$. Therefore
\[
\tilde{A}(\tau) = A(\tau)\, r(\tau)^\alpha
\]
is obtained by multiplying $A(\tau)$ by a non-negative scalar. Hence
$\mathrm{sign}(\tilde{A}(\tau)) = \mathrm{sign}(A(\tau))$.

(2) Consider the case $A(\tau_1) > A(\tau_2) > 0$. By the monotonicity assumption on $r(\cdot)$, we have $r(\tau_1)\ge r(\tau_2)$, and since $x\mapsto x^\alpha$ is non-decreasing on $[0,1]$, this implies
\[
r(\tau_1)^\alpha \ge r(\tau_2)^\alpha.
\]
Thus
\[
\tilde{A}(\tau_1)
= A(\tau_1)\, r(\tau_1)^\alpha
\;\ge\; A(\tau_1)\, r(\tau_2)^\alpha
\;\ge\; A(\tau_2)\, r(\tau_2)^\alpha
= \tilde{A}(\tau_2),
\]
where the last inequality uses $A(\tau_1) > A(\tau_2) > 0$ and $r(\tau_2)^\alpha \ge 0$. Therefore $\tilde{A}(\tau_1)\ge \tilde{A}(\tau_2) > 0$.

Hence, AP-GRPO applies a non-negative reweighting that strictly preserves the ordering of positive advantages.
\end{proof}

\noindent
Within each GRPO group, the AP-GRPO update preserves the direction of relative policy improvement for high-performing trajectories. For trajectories with negative advantages ($A(\tau) < 0$), the weighting term $r(\tau)^\alpha$ acts as a suppression mechanism: as rewards decrease ($r \to 0$), the effective penalty $\tilde{A}(\tau)$ vanishes. This effectively filters out gradients from low-quality ``near-miss'' samples that may contain high variance or irrelevant signals, while focusing the optimization on widening the gap between the best samples and the group mean.

\subsection{Variance Suppression in Sparse Reward Regimes}

In high-precision 3D tasks and other challenging reasoning problems, the model initially generates many ``near-miss'' samples with very low absolute rewards.
Standard GRPO normalizes these low-reward samples to have unit variance within each group, which can introduce significant noise in the policy updates when the absolute reward signal is weak.
AP-GRPO mitigates this by scaling the effective variance of the advantage by the absolute performance.

\begin{theorem}
In sparse reward regimes where the absolute rewards $r_i$ are small, AP-GRPO reduces the variance of the advantage estimator by a factor proportional to $\epsilon^{2\alpha}$, where $\epsilon$ is the mean reward in the group.
\end{theorem}

\begin{proof}
Consider a group of $G$ trajectories, with rewards modeled as
\[
r_i = \epsilon + \xi_i,\quad \mathbb{E}[\xi_i]=0,\quad \mathrm{Var}(\xi_i)=\sigma^2,
\]
where $\epsilon \ll 1$ is a small baseline reward (sparse reward regime), and $\xi_i$ represents small fluctuations around $\epsilon$.
For simplicity, assume $G$ is large enough that the empirical group mean and variance satisfy
\[
\bar r \approx \epsilon,\quad \sigma_r^2 \approx \sigma^2.
\]

In standard GRPO, the within-group advantage is computed as
\begin{equation*}
A_i = \frac{r_i - \bar r}{\sigma_r} \approx \frac{\xi_i}{\sigma_r}.
\end{equation*}
Thus, within the group,
\begin{equation*}
\mathrm{Var}(A_i)
\approx \frac{\mathbb{E}[\xi_i^2]}{\sigma_r^2}
\approx 1,
\end{equation*}
where we used $\sigma_r^2 \approx \sigma^2$.
This unit variance persists even as $\epsilon \to 0$, implying that the magnitude of policy updates remains large, despite the absolute reward signal being tiny.

In AP-GRPO, the advantage is modulated by the absolute term:
\begin{equation*}
A_i^{AP} = A_i\, r_i^\alpha
\approx \frac{\xi_i}{\sigma_r}\,(\epsilon + \xi_i)^\alpha.
\end{equation*}

To approximate $\mathrm{Var}(A_i^{AP})$ for small $\epsilon$, we use a first-order Taylor expansion (Delta method) of $g(x) = x^\alpha$ around $x=\epsilon$:
\begin{equation*}
(\epsilon + \xi_i)^\alpha = \epsilon^\alpha + \alpha \epsilon^{\alpha-1} \xi_i + O(\xi_i^2).
\end{equation*}
Substituting into $A_i^{AP}$, we obtain
\begin{align*}
A_i^{AP}
&= \frac{\xi_i}{\sigma_r} \big(\epsilon^\alpha + \alpha \epsilon^{\alpha-1} \xi_i + O(\xi_i^2)\big) \\
&= \underbrace{\frac{\xi_i}{\sigma_r}\epsilon^\alpha}_{\text{leading term}}
\;+\; \underbrace{\alpha \frac{\xi_i^2}{\sigma_r} \epsilon^{\alpha-1}}_{\text{higher order}}
\;+\; O(\xi_i^3).
\end{align*}

Assuming that $\xi_i$ has finite moments up to order $4$ (a mild regularity condition), the contribution of the higher-order terms to the variance is of smaller order in $\epsilon$.
In particular, for fixed noise scale $\sigma$ and $\alpha \ge 1$, as $\epsilon \to 0$ the leading contribution to the variance comes from the first term:
\begin{equation*}
\mathrm{Var}(A_i^{AP})
= \epsilon^{2\alpha} \mathrm{Var}\Big(\frac{\xi_i}{\sigma_r}\Big)
+ o(\epsilon^{2\alpha})
\approx \epsilon^{2\alpha},
\end{equation*}
where we have used $\mathrm{Var}(\xi_i / \sigma_r) \approx 1$ and absorbed higher-order contributions into $o(\epsilon^{2\alpha})$.

Thus, in the sparse reward regime,
\begin{equation*}
\mathrm{Var}(A^{AP}) \approx \epsilon^{2\alpha} \quad\text{as}\quad \epsilon \to 0.
\end{equation*}
For $\alpha \ge 1$, we have
\begin{equation*}
\lim_{\epsilon \to 0} \mathrm{Var}(A^{AP}) = 0,
\end{equation*}
showing that AP-GRPO suppresses the variance of the advantage estimator in proportion to the absolute reward level.
This effectively down-weights noisy gradient updates when the model has not yet discovered reliably high-reward trajectories.
\end{proof}

\subsection{Recovery of Sensitivity in High-Precision Regimes}
\label{subsec:high_precision}

While the suppression of variance is essential during the early stages of training, the model must eventually recover its sensitivity to subtle differences in reward in order to achieve high-precision alignment.
We now show that as the policy improves and rewards approach the unit maximum ($r \to 1$), AP-GRPO asymptotically recovers the full optimization power of standard GRPO.

\begin{theorem}
In the high-precision regime where the absolute rewards $r_i$ approach $1$, the variance of the AP-GRPO advantage $\mathrm{Var}(A^{AP})$ converges to the variance of the standard GRPO advantage $\mathrm{Var}(A)$, allowing for fine-grained relative optimization.
\end{theorem}

\begin{proof}
Consider a regime where the policy has learned most of the coarse structure of the task, and rewards lie in a small neighborhood of the maximum, i.e., $r_i \in [1-\eta,1]$ for some small $\eta > 0$.
Write
\begin{equation*}
r_i = 1 - \delta_i,\quad \delta_i \ge 0,\quad \delta_i \to 0,
\end{equation*}
where $\delta_i$ represents the residual error of the $i$-th trajectory.

Let $A_i$ denote the standard GRPO advantage within the group, which is centered and normalized so that (up to numerical constants) $\mathbb{E}[A_i] \approx 0$ and $\mathrm{Var}(A_i) \approx 1$.
We assume $A_i$ has finite second moment, i.e., $\mathbb{E}[A_i^2] < \infty$, which is standard in policy gradient analysis.

Under AP-GRPO, the weighted advantage is
\begin{equation*}
A_i^{AP} = A_i\, r_i^\alpha = A_i\, (1 - \delta_i)^\alpha.
\end{equation*}
For small $\delta_i$, a first-order Taylor expansion of $(1 - \delta_i)^\alpha$ around $\delta_i = 0$ yields
\begin{equation*}
(1 - \delta_i)^\alpha \approx 1 - \alpha \delta_i,
\end{equation*}
so that
\begin{equation*}
A_i^{AP} \approx A_i (1 - \alpha \delta_i)
= A_i - \alpha \delta_i A_i.
\end{equation*}

We now examine the variance of $A_i^{AP}$.
We have
\begin{align*}
\mathrm{Var}(A_i^{AP})
&= \mathrm{Var}(A_i - \alpha \delta_i A_i) \\
&= \mathrm{Var}(A_i) + \alpha^2 \mathrm{Var}(\delta_i A_i)
- 2 \alpha\, \mathrm{Cov}(A_i, \delta_i A_i).
\end{align*}
Assume that as training progresses and the policy converges, the residual error $\delta_i$ shrinks in such a way that $\delta_i A_i \to 0$ in mean square, i.e.,
\begin{equation*}
\lim_{t \to \infty} \mathbb{E}[(\delta_i A_i)^2] = 0,
\end{equation*}
where $t$ denotes training time or iteration.
This is a mild regularity condition: as rewards concentrate near $1$, both $\delta_i$ and fluctuations in $A_i$ within the group become small.

Under this assumption,
\begin{equation*}
\lim_{\delta \to 0} \mathrm{Var}(\delta_i A_i) = 0
\quad\text{and}\quad
\lim_{\delta \to 0} \mathrm{Cov}(A_i,\delta_i A_i) = 0.
\end{equation*}
Therefore,
\begin{equation*}
\lim_{r \to 1} \mathrm{Var}(A^{AP})
= \lim_{\delta \to 0} \mathrm{Var}(A_i^{AP})
= \mathrm{Var}(A_i)
\approx 1.
\end{equation*}

Thus, in the high-precision regime where rewards approach $1$, the AP-GRPO advantage recovers the same variance level as standard GRPO.
Intuitively, as $r_i^\alpha \to 1$, the absolute weighting becomes negligible, and AP-GRPO behaves like pure GRPO, allowing the optimizer to fully exploit fine-grained relative differences in reward.
\end{proof}

\subsection{Optimization Dynamics of Dynamic Sharpness Scheduling}

The SNRA operator $\sigma(k, e_i)$ is parameterized by the sharpness $k$. In this section, we provide a theoretical justification for the dynamic scheduling of $k(t)$, demonstrating how it addresses the gradient vanishing problem in the early stages and the precision bottleneck in the late stages of training.

\begin{lemma}[Gradient Extremum Localization]
The gradient of the SNRA reward $\sigma(k, e_i) = \frac{2}{1 + \exp(k e_i)}$ with respect to the numerical error $e_i$ is given by:
\begin{equation*}
\nabla_{e_i} \sigma(k, e_i) = - \frac{2k \exp(k e_i)}{(1 + \exp(k e_i))^2}.
\end{equation*}
The magnitude of this gradient $|\nabla_{e_i} \sigma|$ attains its maximum value when the error $e_i$ satisfies
\begin{equation*}
e_i = 0,
\end{equation*}
at which point the gradient magnitude is $|\nabla_{e_i} \sigma|_{\max} = \frac{k}{2}$.
\end{lemma}

\begin{proof}
We consider the magnitude of the gradient
\[
\bigl|\nabla_{e_i} \sigma(k, e_i)\bigr|
= \frac{2k \exp(k e_i)}{(1 + \exp(k e_i))^2}.
\]
Let $x = \exp(k e_i) > 0$. Then
\[
f(x) = \frac{2k x}{(1 + x)^2}.
\]
Maximizing $|\nabla_{e_i} \sigma|$ over $e_i$ is equivalent to maximizing $f(x)$ over $x>0$. The derivative is
\[
f'(x) = 2k \cdot \frac{(1+x)^2 - x\cdot 2(1+x)}{(1+x)^4}
= 2k \cdot \frac{1 - x^2}{(1+x)^4}.
\]
Setting $f'(x)=0$ gives $1 - x^2 = 0 \Rightarrow x=1$, i.e., $\exp(k e_i)=1$ and hence $e_i=0$.

Evaluating at $e_i=0$ (i.e., $x=1$), we obtain
\[
\bigl|\nabla_{e_i} \sigma(k, 0)\bigr|
= \frac{2k}{(1+1)^2} = \frac{k}{2}.
\]
Thus, the gradient magnitude attains its maximum $\frac{k}{2}$ at $e_i = 0$.
\end{proof}

\begin{theorem}[Global Search via Low Sharpness]
For large initial errors $e_i \gg 0$, a high sharpness value $k$ leads to exponential gradient vanishing, whereas a low sharpness $k$ preserves the learning signal.
\end{theorem}

\begin{proof}
Consider an initial error $E > 0$. The gradient magnitude at $E$ for a given $k$ is
\begin{equation*}
\bigl|\nabla_{e_i} \sigma(k, E)\bigr|
= \frac{2k \exp(k E)}{(1 + \exp(k E))^2}
= \frac{2k}{\exp(k E) + 2 + \exp(-k E)}.
\end{equation*}
When $k \to \infty$ for a fixed $E > 0$, the denominator grows exponentially as $O(\exp(k E))$, while the numerator grows only linearly as $O(k)$. Thus:
\begin{equation*}
\lim_{k \to \infty} \bigl|\nabla_{e_i} \sigma(k, E)\bigr| = 0.
\end{equation*}
In other words, a fixed high-sharpness reward would fail to provide any meaningful gradient for samples that are far from the target.

By initializing the sharpness at a small value $k_{\min}$ at $t=0$, we keep $\exp(k_{\min} E)$ at a moderate scale even for large $E$, thereby avoiding exponential suppression of the gradient. This maintains a non-vanishing learning signal that enables the model to begin orienting its predictions toward the ground truth from a wide range of initial errors.
\end{proof}

\begin{theorem}[Precision Refinement via High Sharpness]
As the model converges ($e_i \to 0$), the reward surface becomes effectively flat for low $k$, leading to insufficient optimization pressure for high-precision tasks. Increasing $k$ restores useful gradient magnitudes in this regime.
\end{theorem}

\begin{proof}
Let $\delta > 0$ represent a small residual error in the late stage of training. The gradient magnitude at $e_i = \delta$ is
\begin{equation*}
\bigl|\nabla_{e_i} \sigma(k, \delta)\bigr|
= \frac{2k \exp(k\delta)}{(1 + \exp(k\delta))^2}.
\end{equation*}
For $k\delta \ll 1$, a Taylor expansion $\exp(k\delta) = 1 + k\delta + O((k\delta)^2)$ gives
\begin{equation*}
\bigl|\nabla_{e_i} \sigma(k, \delta)\bigr|
= \frac{2k (1 + k\delta + O((k\delta)^2))}{(2 + k\delta + O((k\delta)^2))^2}
= \frac{k}{2} + O(k^2 \delta).
\end{equation*}
Thus, near $e_i=0$, the gradient magnitude is approximately $\frac{k}{2}$ and does not vanish as $\delta \to 0$; instead, it is controlled primarily by $k$.

However, the \emph{reward difference} between a nearly correct prediction and a perfect one is
\begin{equation*}
|\sigma(k,\delta) - \sigma(k,0)|
= \left|\frac{2}{1+\exp(k\delta)} - 1\right|
= \frac{2}{1+\exp(k\delta)} - 1,
\end{equation*}
where the last equality uses $\sigma(k,0)=1$. For $k\delta \ll 1$, we have
\begin{align}
\exp(k\delta) &= 1 + k\delta + O((k\delta)^2), \\
\frac{2}{1+\exp(k\delta)}
&= \frac{2}{2 + k\delta + O((k\delta)^2)}
= 1 - \frac{k}{2}\delta + O((k\delta)^2),
\end{align}
so that
\begin{equation*}
|\sigma(k,\delta) - \sigma(k,0)|
\approx \frac{k}{2}\,\delta \quad \text{for } k\delta \ll 1,
\end{equation*}
which is $O(k\delta)$.

For a fixed low $k$, as $\delta \to 0$ this difference becomes extremely small and can be easily overshadowed by numerical noise or by regularization terms such as the KL-divergence penalty $\beta_{\mathrm{KL}}$. In other words, with small $k$ the reward surface becomes too flat to reliably distinguish between errors of size $\delta$ and $0$ at high precision.

By increasing $k(t)$ to a larger value $k_{\max}$ in the late stage of training, both the local slope near $e_i=0$ (which scales like $k/2$) and the reward contrast between $e_i=\delta$ and $e_i=0$ (which scales like $k\delta$) are amplified by orders of magnitude. This compensates for the diminishing residual error $\delta$, ensuring that the effective gradient signal on the policy parameters $\nabla_\theta \mathcal{L}$ remains large enough to continue driving the model toward the exact optimum $e_i = 0$.
\end{proof}

This coarse-to-fine curriculum guarantees that the model is neither stalled by gradient vanishing at the beginning nor limited by insufficient gradient resolution at the end.

\section{Extended Ablation Studies and Sensitivity Analysis}
\label{appendix:ablation}

To further validate the robustness and internal mechanisms of the \textit{Smooth Operator} framework, we provide a series of controlled experiments focusing on hyperparameter sensitivity and operator design.

\subsection{Dynamics of Sharpness Schedule ($k_{min}$ and $k_{max}$)}

The SNRA operator relies on a dynamic sharpness schedule. We analyze the impact of the starting point ($k_{min}$) and the terminal bound ($k_{max}$) separately.

\paragraph{Initial Smoothness $k_{min}$.} 
As shown in Table \ref{tab:kmin_sens}, the framework is highly robust to $k_{min}$ within the $(0, 1]$ range. However, performance degrades when $k_{min} > 1$, as a prematurely sharp reward surface creates a "cold start" problem, penalizing early-stage exploration.

\begin{table}[h]
\centering
\caption{Sensitivity analysis of initial $k_{min}$ on VSI-Bench accuracy.}
\label{tab:kmin_sens}
\begin{tabular}{lccccc}
\toprule
\textbf{Hyperparameter} & \multicolumn{5}{c}{\textbf{Varying Initial $k_{min}$}} \\ 
\cmidrule(lr){2-6}
Value & 0.1 & 0.5 & \textbf{1.0 (Default)} & 2.0 & 5.0 \\ 
\midrule
VSI-Bench Acc. (\%) & 59.8 & 60.1 & \textbf{60.0} & 57.2 & 54.5 \\ 
Training Status & Stable & Stable & \textbf{Stable} & Unstable & Degenerated \\
\bottomrule
\end{tabular}
\end{table}

\paragraph{Terminal Sharpness $k_{max}$.}
Table \ref{tab:kmax_sens} illustrates a parabolic performance trend. While increasing $k$ enhances precision, an excessively large $k_{max}$ (e.g., 400) leads to reward sparsity and high-variance gradients. $k_{max} \approx 100$ serves as the optimal "sweet spot" for fine-grained 3D reasoning.

\begin{table}[h]
\centering
\caption{Sensitivity analysis of peak sharpness $k_{max}$ on VSI-Bench accuracy.}
\label{tab:kmax_sens}
\begin{tabular}{lccccc}
\toprule
\textbf{Hyperparameter} & \multicolumn{5}{c}{\textbf{Varying $k_{max}$}} \\ 
\cmidrule(lr){2-6}
Value & 25 & 50 & \textbf{100 (Default)} & 200 & 400 \\ 
\midrule
VSI-Bench Acc. (\%) & 54.8 & 57.6 & \textbf{60.0} & 59.1 & 57.9 \\ 
$\Delta$ Performance & -5.2 & -2.4 & - & -0.9 & -2.1 \\
\bottomrule
\end{tabular}
\end{table}

\subsection{Sensitivity of Absolute Scaling Factor ($\alpha$)}

We investigate the impact of the integer-scaled factor $\alpha$ in AP-GRPO. Since $R_i \in [0, 1]$, the absolute reward component $R_i^\alpha$ diminishes quadratically when $\alpha=2$. This leads to \textit{advantage vanishing}, where the gradients become too small to guide the model beyond the SFT baseline (see Table \ref{tab:alpha_sens}).

\begin{table}[h]
\centering
\caption{Impact of scaling factor $\alpha$ on training stability.}
\label{tab:alpha_sens}
\begin{tabular}{lccc}
\toprule
\textbf{Scaling Factor} & $\alpha = 0$ (GRPO) & $\alpha = 1$ (AP-GRPO) & $\alpha = 2$ \\
\midrule
Final Accuracy (\%) & 54.4 & \textbf{60.0} & 49.7 \\
Training Status & Stable & \textbf{Optimal} & \textbf{Vanishing} \\
Gradient Norm & Normal & Normal & $\approx 0$ \\
\bottomrule
\end{tabular}
\end{table}

\subsection{Ablation of Smooth Numerical Reward Activation}We compare the Sigmoid-based operator with a Tanh-based implementation to ensure the framework is not dependent on a specific functional form. To align the Tanh operator with the SNRA requirement---where the reward reaches its maximum of 1 \textbf{only when the error is zero} ($e_i=0$) and decays to 0 as the error increases---we employ a normalized Tanh formulation:\begin{equation}r_i = 1 - \tanh(k \cdot e_i), \quad \text{where } e_i \in [0, \infty], r_i \in [0, 1]\end{equation}In this formulation, when $e_i = 0$, $r_i = 1 - \tanh(0) = 1$. As $e_i \to \infty$, $r_i$ asymptotically approaches $0$ (since $\tanh(\infty) = 1$). This ensures that the Tanh-based operator maintains the same $[0, 1]$ range and monotonic decay as the Sigmoid version. Experimental results (Table \ref{tab:act_ablation}) show that the choice between these smooth operators has a negligible effect on final accuracy.

\begin{table}[h]
\centering
\caption{Performance comparison of Sigmoid and Shifted-Tanh operators.}
\label{tab:act_ablation}
\begin{tabular}{lc}
\toprule
\textbf{Operator Type} & \textbf{Accuracy (\%)} \\
\midrule
Sigmoid (Ours) & \textbf{59.97} \\
Tanh (Shifted) & 59.89 \\
\bottomrule
\end{tabular}
\end{table}

The marginal 0.08\% difference confirms that the smooth-to-binary transition logic is the key driver of the performance gains, rather than the specific mathematical choice of the activation function.

```latex

\section{Discrete Verification Scores and the Scaling Function $\Phi$}
\label{sec:phi_discrete_detailed}

This section details how we construct (i) task-specific verifier scores $V_\tau(o)\in[0,1]$ that provide partial credit for discrete relational subtasks, and (ii) a scaling function $\Phi_\tau$ that maps these scores into a unified non-negative error space compatible with SNRA.

\subsection{From Discrete Scores to a Unified Error}
\label{subsec:disc_to_error}

For a discrete subtask $\tau\in\mathcal{T}_{\mathrm{disc}}$, the verifier produces a graded score $V_\tau(o)\in[0,1]$ for a model output $o$.
We convert this score into a \emph{logical miss}:
\begin{equation}
u \triangleq 1 - V_\tau(o)\in[0,1],
\end{equation}
and map it into an error scalar:
\begin{equation}
e_{\mathrm{disc}} = \Phi_\tau(u)\in[0,+\infty).
\end{equation}
The mapping is designed to satisfy:
\begin{equation}
\Phi_\tau(0)=0,\quad \lim_{u\to 1}\Phi_\tau(u)=+\infty,\quad \Phi_\tau\ \text{is monotone increasing}.
\end{equation}
The first condition ensures that a fully correct discrete prediction ($V_\tau=1$) corresponds to zero error; the divergence at $u\to 1$ ensures that complete failures ($V_\tau\to 0$) can be mapped to arbitrarily large penalties; the monotonicity ensures that higher verification scores always yield smaller errors.

We define a unified error space $E=[0,\infty)$. In practice, we employ numerical stabilizers (e.g., $\epsilon$-clipping for logs and $e_{\max}$ for invalid parses) which may upper-bound realized error values while preserving non-negativity and the intended ordering.

\subsection{Log-Scaled $\Phi_\tau$ and Calibration}
\label{subsec:phi_form_calib}

We adopt a log-based unbounded scaling:
\begin{equation}
\label{eq:phi_log_detailed}
\Phi_\tau(u)=\eta_\tau\left[-\log\left(1-u+\epsilon\right)\right]^{\gamma_\tau},
\qquad u\in[0,1),
\end{equation}
where $\eta_\tau>0$ is a scale factor, $\gamma_\tau>0$ controls curvature, and $\epsilon\in(0,1)$ is a numerical stabilizer.
Since $-\log(1-u+\epsilon)\ge 0$ and grows without bound as $u\to 1$ (for small $\epsilon$), $\Phi_\tau$ maps $u\in[0,1)$ to $[0,+\infty)$, meeting the unified error-domain requirement in theory.

In implementation, we clip $u$ to avoid evaluating $\log(0)$:
\begin{equation}
u \leftarrow \min\left(u,\,1-\epsilon\right),
\end{equation}
which yields a finite but very large error for exact failures, while preserving numerical stability.

To ensure that (near-)failures fall into the high-penalty region of SNRA at the end of training, we calibrate $\eta_\tau$ using the terminal sharpness $k_{\max}$ and a target near-zero reward level $\epsilon_r\in[10^{-3},10^{-2}]$.
Recall SNRA:
\begin{equation}
\hat{\sigma}(k,e)=\frac{2}{1+\exp(ke)}.
\end{equation}
We first compute the target error $e^\star$ that corresponds to reward $\epsilon_r$ under $k_{\max}$:
\begin{equation}
\label{eq:estar_from_epsr}
\hat{\sigma}(k_{\max},e^\star)=\epsilon_r
\;\Rightarrow\;
e^\star=\frac{1}{k_{\max}}\ln\left(\frac{2}{\epsilon_r}-1\right).
\end{equation}
Let $V_{\min}$ denote the score floor treated as ``complete failure'' (typically $V_{\min}=0$), so that $u_{\mathrm{fail}}=1-V_{\min}$.
With the clipped implementation, this becomes:
\begin{equation}
\label{eq:u_fail_clip}
u_{\mathrm{fail}}=\min\left(1-V_{\min},\,1-\epsilon\right)=1-\epsilon.
\end{equation}
We calibrate $\eta_\tau$ by enforcing $\Phi_\tau(u_{\mathrm{fail}})=e^\star$:
\begin{equation}
\label{eq:eta_calib}
\eta_\tau
=
\frac{e^\star}
{\left[-\log\left(1-u_{\mathrm{fail}}+\epsilon\right)\right]^{\gamma_\tau}}
=
\frac{e^\star}
{\left[-\log\left(2\epsilon\right)\right]^{\gamma_\tau}}.
\end{equation}
This guarantees that a score at the failure floor is mapped to an SNRA reward approximately $\epsilon_r$ when $k(t)$ reaches $k_{\max}$.

Unless stated otherwise, we share a single $\eta$ across all discrete tasks to keep the reward scale consistent, and tune $\gamma_\tau$ mildly across task types (typically $\gamma_\tau\in[1,2]$).

\subsection{Verifier Designs with Partial Credit}
\label{subsec:verifier_design_detailed}

We now specify $V_\tau(o)$ for four discrete subtasks: Relative Direction, Appearance Order, Counting, and Position Relationship.
Each verifier is deterministic and checkable, while providing partial credit for near-miss predictions.

\subsubsection{Relative Direction}
\label{subsubsec:verifier_direction}

The goal is to predict the direction of object $A$ relative to $B$ (e.g., 4-way or 8-way direction bins).
Let the ground-truth class be $c$ and the predicted class be $\hat{c}$, both mapped to indices $i,j\in\{0,\ldots,K-1\}$, where $K\in\{4,8\}$.
We define the circular distance on the direction ring:
\begin{equation}
d(c,\hat{c}) = \min(|i-j|,\,K-|i-j|).
\end{equation}
We then assign a graded score:
\begin{equation}
\label{eq:V_dir_detailed}
V_{\mathrm{dir}}(o)=
\begin{cases}
1, & d(c,\hat{c})=0,\\
\eta_{\mathrm{dir}}, & d(c,\hat{c})=1,\\
0, & d(c,\hat{c})\ge 2,
\end{cases}
\end{equation}
where $\eta_{\mathrm{dir}}\in(0,1)$ is a fixed near-miss credit (we use $\eta_{\mathrm{dir}}=0.5$ by default).
This scheme gives partial credit to predictions that are off by one adjacent direction bin, while assigning zero credit to larger angular mistakes.

When continuous geometry is available (optional), we can compute the ground-truth relative bearing angle $\theta$ and the predicted bearing $\hat{\theta}$, and use a smooth angular score:
\begin{equation}
\Delta \triangleq \mathrm{wrap}(|\theta-\hat{\theta}|)\in[0,\pi],\qquad
V_{\mathrm{dir}}(o)=\exp\left(-\frac{\Delta^2}{2\sigma_\theta^2}\right),
\end{equation}
where $\sigma_\theta$ controls angular tolerance, yielding fully continuous partial credits.

\subsubsection{Appearance Order}
\label{subsubsec:verifier_order}

This task evaluates whether the model correctly predicts the appearance/visibility order of objects.
We support both pairwise and listwise formulations.

For a pairwise query on objects $A$ and $B$, the verifier extracts their first-visible indices (or timestamps) $t_A$ and $t_B$ from the annotated sequence.
Let the predicted relation be $\widehat{\mathrm{ord}}\in\{A\prec B,\,B\prec A\}$.
The ground-truth relation is induced by $\mathrm{sign}(t_A-t_B)$.
We define:
\begin{equation}
\label{eq:V_order_pair_detailed}
V_{\mathrm{ord}}(o)=
\mathbb{I}\!\left[\widehat{\mathrm{ord}}=\mathrm{sign}(t_B-t_A)\right]
\cdot
\left(1-\exp\left(-\frac{|t_A-t_B|}{\beta}\right)\right),
\end{equation}
where $\beta>0$ sets a margin scale.
The term $1-\exp(-|t_A-t_B|/\beta)$ down-weights nearly-tied cases, reflecting that these are intrinsically ambiguous even when the predicted order matches.

For listwise ordering with $n$ objects, the verifier compares the predicted permutation $\hat{\pi}$ against the ground-truth permutation $\pi$ using the Kendall inversion count $K(\pi,\hat{\pi})$.
The normalized score is:
\begin{equation}
\label{eq:V_order_list_detailed}
V_{\mathrm{ord}}(o)=1-\frac{K(\pi,\hat{\pi})}{\binom{n}{2}},
\end{equation}
which lies in $[0,1]$ and provides partial credit proportional to how many pairwise order constraints are satisfied.

\subsubsection{Counting}
\label{subsubsec:verifier_count}

The counting task requires predicting an integer $\hat{n}$ for a ground-truth count $n$.
We parse $\hat{n}$ from the model output and define the absolute deviation:
\begin{equation}
\Delta \triangleq |\hat{n}-n|.
\end{equation}
We use an exponentially decaying correctness score:
\begin{equation}
\label{eq:V_count_detailed}
V_{\mathrm{cnt}}(o)=\exp\left(-\frac{\Delta}{\tau}\right),
\end{equation}
where $\tau>0$ controls tolerance.
With $\tau=1$, an off-by-one prediction receives $e^{-1}$ partial credit, and larger errors decay smoothly, avoiding binary sparsity.

When count ranges vary significantly across tasks (optional), we can use a normalized linear score:
\begin{equation}
V_{\mathrm{cnt}}(o)=\max\left(0,\,1-\frac{|\hat{n}-n|}{\max(n,1)+c}\right),
\end{equation}
where $c>0$ prevents instability when $n$ is small.

\subsubsection{Position Relationship}
\label{subsubsec:verifier_position}

This task evaluates spatial relations between objects (e.g., \texttt{left-of}, \texttt{in-front-of}, \texttt{inside}, \texttt{overlap}, \texttt{near}).
Since multiple relations can simultaneously hold, we treat this as a multi-label verification problem whenever applicable.

Let $S$ be the ground-truth set of relations and $\hat{S}$ be the predicted set extracted from the output.
We use Jaccard similarity as the verifier score:
\begin{equation}
\label{eq:V_pos_jaccard_detailed}
V_{\mathrm{pos}}(o)=\frac{|\hat{S}\cap S|}{|\hat{S}\cup S|}.
\end{equation}
This yields partial credit for predicting a correct subset, penalizes missing relations, and also penalizes spurious relations.

If a dataset defines mutually exclusive relation labels, we can optionally define a relation graph with shortest-path distance $d(c,\hat{c})$ between the predicted and ground-truth labels, and use:
\begin{equation}
V_{\mathrm{pos}}(o)=1-\frac{d(c,\hat{c})}{d_{\max}},
\end{equation}
where $d_{\max}$ is the maximum graph distance. This assigns partial credit when the predicted relation is semantically close to the correct one.

\subsection{Final Discrete Error for SNRA}
\label{subsec:final_disc_error_detailed}

Given $V_\tau(o)$ from the task-specific verifiers above, we compute $u=1-V_\tau(o)$ and apply the log-scaled mapping:
\begin{equation}
\label{eq:disc_error_final_log}
e_{\mathrm{disc}} = \Phi_\tau(1-V_\tau(o))
=
\eta_\tau\left[-\log\left(V_\tau(o)+\epsilon\right)\right]^{\gamma_\tau}.
\end{equation}
This mapping is non-negative, monotone decreasing in $V_\tau$, and unbounded as $V_\tau(o)\to 0$, hence $e_{\mathrm{disc}}\in[0,+\infty)$ in theory.

The resulting $e_{\mathrm{disc}}$ is fed into SNRA to produce the smooth numerical reward for discrete tasks:
\begin{equation}
\tilde{r}=\hat{\sigma}(k(t), e_{\mathrm{disc}}).
\end{equation}
With the calibration in Eq.~\eqref{eq:estar_from_epsr}--\eqref{eq:eta_calib}, fully incorrect (or near-zero) discrete predictions receive near-zero rewards at late training (high $k$), while partially correct predictions remain within gradient-sensitive regions, improving sample utilization and stabilizing AP-GRPO updates.

\section{Case Visualization}

\begin{figure}
    \centering
    \includegraphics[width=1\linewidth]{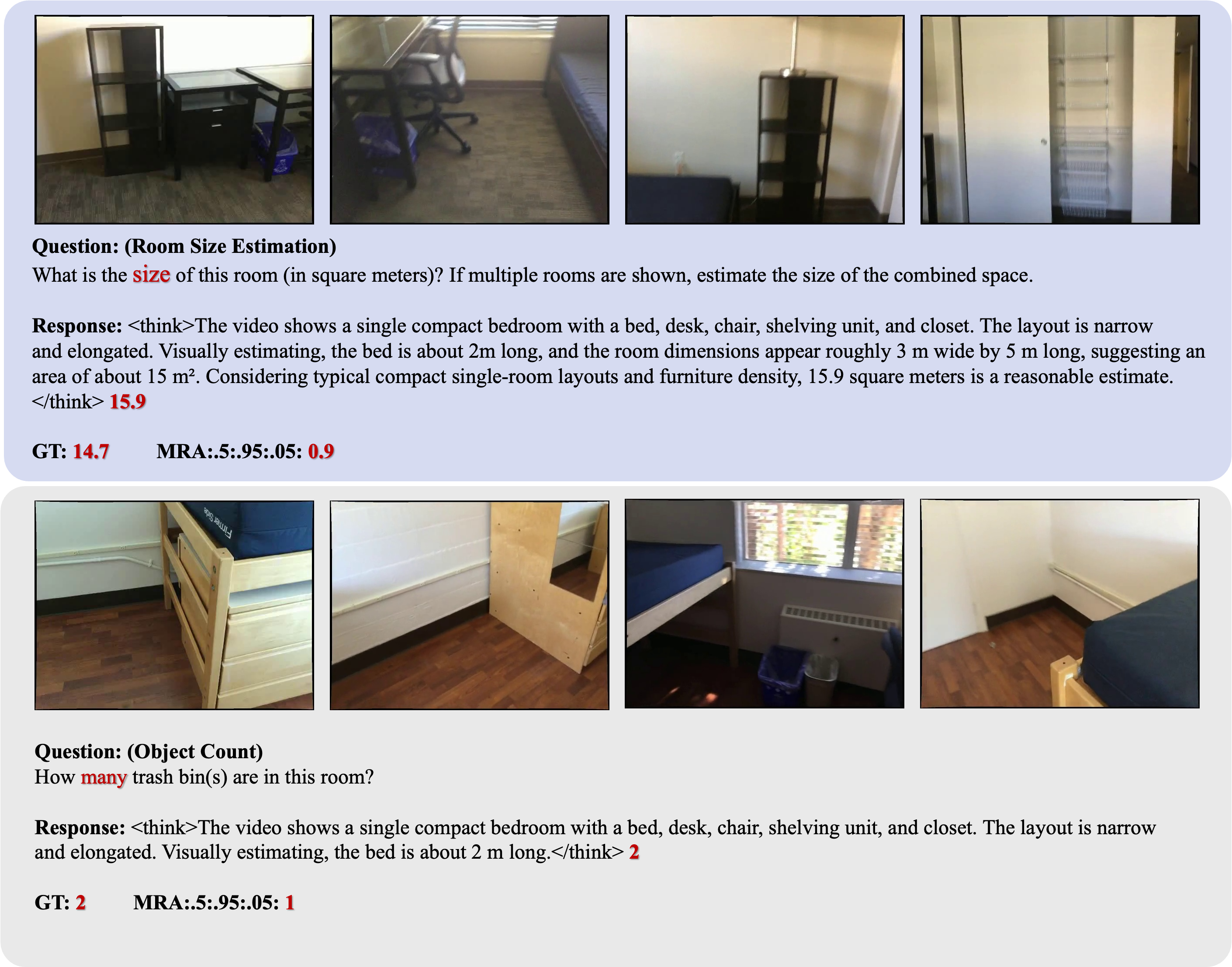}
    \caption{Cases of VSI-Bench}
    \label{fig:placeholder}
\end{figure}

\begin{figure}
    \centering
    \includegraphics[width=1\linewidth]{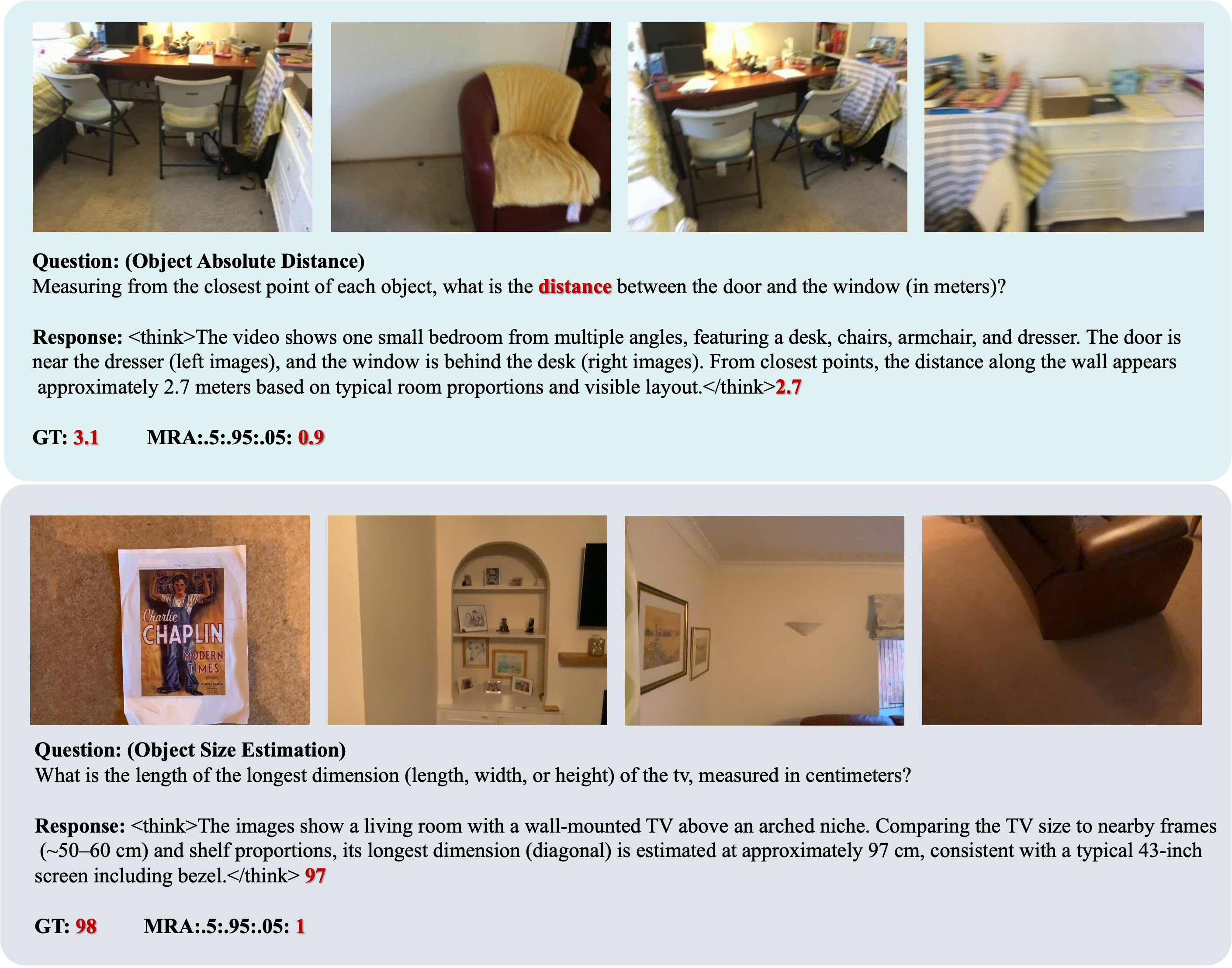}
    \caption{Cases of VSI-Bench}
    \label{fig:placeholder}
\end{figure}

\end{document}